\documentclass[twocolumn]{paper}    
\usepackage{times}
\usepackage{booktabs}
\usepackage{amsmath,rotating}
\usepackage{amsfonts}
\usepackage{amssymb}
\usepackage{times,multirow,hhline,subfigure}
\usepackage{comment}
\usepackage{algorithmic}
\usepackage{algorithm}
\usepackage{graphicx}
\usepackage{subfigure}
\usepackage[scanall]{psfrag}
\usepackage{theorem}
\newtheorem{Theorem}{Theorem}

\newtheorem{Lemma}{Lemma}

\newcommand{\eq}{\begin{equation}}
\newcommand{\eqn}{\end{equation}}
\newcommand{\abstr}{\begin{quotation}\sl}
\newcommand{\abstrn}{\end{quotation}}

\def\sqw{\hfill\hbox{\lower.1ex\hbox{$\sqcup$}
    \kern-1.02em\lower.1ex\hbox{$\sqcap$}}\ }

\newtheorem{Definition}{Definition}
\newcommand{\bmat}[1]{ \begin{bmatrix} #1 \end{bmatrix} }
\newenvironment{proof}{\noindent{\bf Proof:}}{}
\DeclareMathOperator*{\argmin}{\arg\min}

\DeclareMathOperator*{\diag}{diag}

\DeclareMathOperator*{\sign}{sign}

\newcommand{\R}{\mathbb{R}}

\newcommand{\U}{\mathbf{u}}
\newcommand{\Obs}{\mathcal{S}}

\begin{document}
\title{MINLIP for the Identification of Monotone Wiener Systems}
\author{Kristiaan Pelckmans |kp@it.uu.se|,\\ 
		{\small Division of Systems and Control, Department of Information Technology}\\{\small
		Uppsala University, Box 337, SE-751 05, Uppsala, Sweden}}
%
%
%
%
%
%
%
%
%
%
\maketitle
\thispagestyle{empty}
\pagestyle{empty}

\begin{abstract}
	This paper studies the MINLIP estimator for the identification of Wiener systems consisting of 
	a sequence of a linear FIR dynamical model, and a monotonically increasing (or decreasing) static function. 
	Given $T$ observations, this algorithm boils down to solving a convex 
	quadratic program with $O(T)$ variables and inequality constraints, 
	implementing an inference technique which is based entirely on model complexity 
	control\footnote{This is a rewritten, refined and extended version of the CDC 2010 
	submission 'On the Identification of Monotone Wiener Systems'.}.
	The resulting estimates of the linear submodel are found to be almost consistent when no noise is present in the data,
	under a condition of smoothness of the true nonlinearity and local Persistency of Excitation (local PE) of the data.
	This result is novel as it does not rely on classical tools as a 'linearization' using a Taylor decomposition, 
	nor exploits stochastic properties of the data.
	It is indicated how to extend the method to cope with noisy data, 
	and empirical evidence contrasts performance of the 
	estimator against other recently proposed techniques. 
\end{abstract}


\section{INTRODUCTION}

The identification of Wiener systems has been considered in many papers since the 1970s.
Different existing approaches could roughly be divided in methods using 
(i) Invertible nonlinearities (reducing to Hammerstein identification), 
(ii) correlation based approaches exploiting stochastic properties of the signals \cite{billings1978,bai2002},
(iii) approximate (recursive) PEM approaches providing a well-established framework for convergence analysis 
\cite{wigren1994,wigren1995} and \cite{hagenblad99a},
(iv) subspace based approaches \cite{westwick1996}.
For a general overview see the survey \cite{greblicki2008}.
Specific applications towards identification with quantized outputs
are considered in \cite{wigren1998}, see also the book \cite{tsumura2002}.
The present approach builds further on ideas developed in \cite{zhang2006,bai2008}.
Another relevant work is \cite{voros2001} which considers a similar identification task as we will do.
The MINLIP estimator (or shortly MINLIP, we will clarify the abbreviation shortly) 
studied in this paper for identification of dynamic Wiener systems 
was originally conceived in the context of learning ranking  functions and survival analysis, 
see \cite{vanbelle09c} and earlier work of those authors. 
In \cite{pelckmans05_hamcdc}, the authors studied the impact of model 
complexity control for the identification of Hammerstein systems.
In \cite{pelckmans10a} the use of explicit complexity control 
was investigated in the context of adaptive filtering.

While the literature on the identification of Wiener systems is considerable, 
often theoretical understanding of the proposed techniques is restricted to exposition of an appropriate technical implementation.
Notable exceptions are given in \cite{wigren1995}, \cite{greblicki2008} and \cite{bai2008}.
The first work considers a Recursive Prediction Error Method (RPEM) of general Wiener models, 
and convergence properties are derived using the ODE framework as in  \cite{ljung1977}.
This approach makes considerable assumptions on the stochastic mechanisms underlying the 
signals used for (recursive) identification.
The smoothing approach described in \cite{greblicki2008} exploits as well 
a stochastic assumption of the involved signals, and asserts basically that the Wiener 
system can be identified by directly averaging out the nonlinear effect. 
Although powerful concentration inequalities lie on the basis of this approach,
no argument is given that this method applies for Wiener systems which are more complex
(realistic) than the academic examples presented in those papers.

This work was prompted by the earlier \cite{zhang2006}, exploring the task of identification of 
monotone wiener systems. 
The practical algorithm proposed in that paper will always yield a trivial solutions ($h=0_d$ in their notation),
and is as such to be depreciated. The line of thinking however looks powerful, and this led
\cite{bai2008} to investigate the question under what conditions on 
the static nonlinearity (besides monotonicity) a Wiener system is identifiable. 
The present work takes this results a step further, introducing model complexity control into the picture.
Specifically, we express model complexity in terms of a Lipschitz property of the estimated nonlinearity,
The idea of minimizing (MIN-) this Lipschitz property (-LIP)  results directly in an efficient identification technique
termed the MINLIP estimator  - or shortly MINLIP -
which can be solved efficiently using tools of convex optimization.
Specifically, we rephrase the identification task as a convex Quadratic Programming (QP) problem of 
$O(T)$ unknowns and inequalities (where $T$ denotes the number of samples).
We as well present an analysis that the estimates given by this algorithm are almost consistent,
where the approximation factor relies on the {\em richness} of the data in terms of a local measure of persistency of excitation, 
and the smoothness of the static nonlinearity.
This analysis does not resort to local approximations using a Taylor decomposition, and does not need 
any stochastic setup.

The contribution of this paper is threefold.
Section 2 describes the precise class of monotone Wiener systems which is 
envisaged, and discusses the main ideas motivating the MINLIP estimator.
An artificial yet challenging case study provides empirical evidence for this estimator for noiseless data. 
Section 3 then establishes almost consistency of the estimates under appropriate 
conditions of the data used for identification, and the underlying model.
This result is non-trivial as the considered model does not allow for 
a straightforward (finite) parametrization, and is not based on minimizing a model mismatch criterion as classical.
In particular, we need to have that the data is locally Persistent Exciting (PE) while 
the true monotone nonlinearity needs to be smooth around its steepest part.

Section 3 describes an extension towards the case where
(i) data is perturbed by noise, 
(ii) the true system does not belong to the considered model class of monotone Wiener systems of given order, or 
(iii) where  no true model is assumed to exist.
Again, the MINLIP estimates are given by solving a convex Quadratic Program with $O(T)$ unknowns and inequality constraints.
Empirical evidence is given for the use of this estimator, and 
the degradation of the accuracy of the estimate in case of noise is investigated experimentally.
Section 5 concludes the paper, and highlights a number of open research questions.

\section{Identification of Monotone Wiener Systems}

\subsection{The Model Class}

\begin{figure}
	\begin{center}\includegraphics[height=1.2cm]{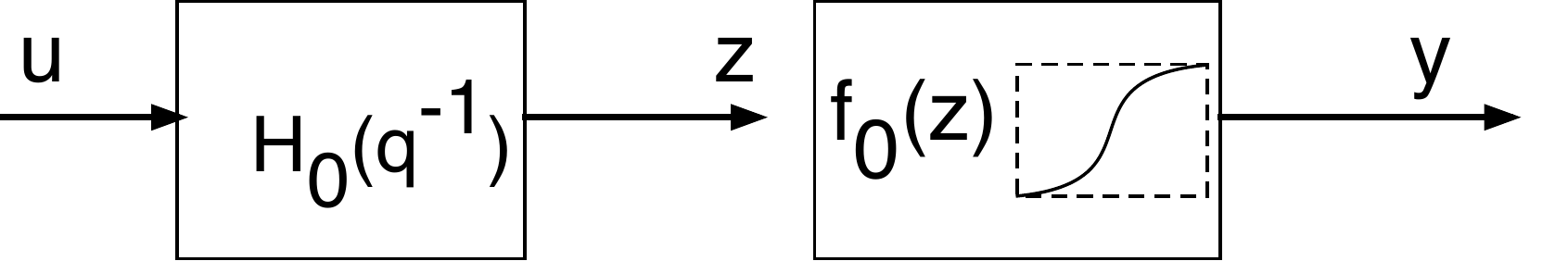}\end{center}
  \caption{\sl Schematic  representation of Wiener systems under consideration. 
  	The function $f_0:\R\rightarrow\R$ is assumed to be monotonically in- or decreasing. 
	Neither $H_0$ nor $f_0$ is assumed to be invertible.
  }
  \label{fig.wiener}
\end{figure}

This work focus on the identification of nonlinear dynamic models 
in the following model class.

\begin{Definition}[FIR Wiener Model $(f,a)$]
  A Wiener model consists of a sequence of 
  (i) a linear dynamical model characterized by an impulse response function $H(q^{-1})$ 
  (here $q^{-1}$ is the backshift operator as classically) 
  applied on the {\em input} signal $\{u_t\}_t$,
  and (ii) a static nonlinear function $f:\R\rightarrow\R$ (see Fig. 1).
  If the signals $\{u_t\}_t$ and $\{y_t\}_t$ follow such a model with 
  'true' subsystem $H_0$ and 'true' function $f_0$, we can write 
  \begin{equation}
  	y_t = f_0\Big( H_0(q^{-1})(u_t) \Big),
  	\label{eq.wiener}
  \end{equation}
  and we say that the observations come from the Wiener {\em system} $(H_0,f_0)$.
  For a FIR-Wiener model of order $d$, 
  one considers a Finite Impulse Response (FIR) parametrization of the linear subsystem, or 
  $H(q^{-1}) = a_1 q^{-1} + \dots + a_dq^{-d}$.
  We denote such model (in the context of this paper) shortly as the $(f,a)$-Wiener model.
  Now since $a$ and the domain of $f$ can be rescaled arbitrarily, it is convenient to impose $\|a\|_2=1$,
  which does avoid identifiabilty issues.
  If the signals $\{u_t\}_t$ and $\{y_t\}_t$ obey such a model with 'true' function $f_0$ 
  and 'true' parameters $a_0$, or 
  \begin{equation}
  		y_t 	= f_0\left( \sum_{k=1}^{d} a_{0,k} u_{t-k}\right)
    			= f_0( a_0^T \U_t ).
    \label{eq.wiener.fir}
  \end{equation}
  - where $a_0\in\R^d$ and $\|a_0\|_2=1$  and 
  we define $\U_t=(u_{t-1},\dots, u_{t-d})^T\in\R^d$ -
  we say that the observations come from the Wiener {\em system} $(a_0,f_0)$.
  We will denote the set of possible observations as 
  $\Obs = \{(\U_t,y_t)\}_t\subseteq\R^d\times\R$.
\end{Definition}
We specialize further to a subset of this class as follows,
schematically illustrated in fig. (\ref{fig.wiener}).
\begin{Definition}[Monotone FIR Wiener Model $(f,a)$]
	A FIR-Wiener model $(f,a)$ is called {\em monotone} if 
	$f:\R\rightarrow\R$ is monotonically increasing (but not necessarily invertible),
	 \begin{equation}
    		y_t 	= f( a^T \U_t ).
		\label{eq.monowien}
	\end{equation}
	We define the Monotone Wiener model class formally as 
	\begin{multline}
		\mathcal{F} 
		= \left\{ (f,a) \ \Big| \ f:\R\rightarrow\R: \mbox{Monotonically increasing}, \right.\\ \left. 
		a\in\R^d, \|a\|_2=1\right\}.
		\label{eq.class}
	\end{multline}
\end{Definition}
Note that by similarity one has $(f,a)= (f',-a)$, where $f(z) = f'(-z)$ for all $z\in\R$.
Now $f'$ is monotonically decreasing, explaining why we can omit the denominator 'increasing' in the nomenclature.
As argued before, this class of monotone FIR-Wiener model can capture such different effects as 
(1) quantized output measurements, 
(2) saturation effects of the sensor, and 
(3) handling of general bijective transformations of the output scaling (cfr. the temperature scale of Celsius versus Fahrenheit), amongst others.

\subsection{Identification by MINLIP}

The identification technique implements the adagio 'make problems as simple as possible, but not simpler'. 
The surprising result is that this idea may yield consistent estimates, without any reference 
to notions as 'statistical likelihood' or 'prediction error'.

The problem of identification of a Wiener system from observations is traditionally formalized as follows
\begin{equation}
  \min_{a,f: \: \|a\|_2=1} J(f,a) = \sum_{t=d+1}^{T} \left( f(a^T\U_t) -  y_t \right)^2.
  \label{eq.whcost}
\end{equation}
We refer to this formulation of the estimation problem as to a prediction error method for Wiener 
models - abbreviated here as WPEM - and it will be mainly this approach we will contrast 
the proposed method against.
Note that this approach in case of noise as in Section III is a merely an Output-Error (OE) technique,
unless stringent stochastic assumptions can be made on the noise, see e.g. \cite{hagenblad99a} for a discussion.
In general,this formulation is hard to solve as the unknowns $a$ interact directly with the unknown function $f$. As a result one typically resorts to an iterative scheme or general purpose nonlinear optimization routine. The practical procedures lack generality and robustness for different reasons
(i) depending on the form (or parametrization) of $f$, ill-conditioning of the problem may arise or even gradient information may not exist,
(ii) the problem can often be stuck in local minima, which can be arbitrary bad
(iii) procedures are highly depending on the exact representation of the unknown $f$.
In general, such procedures are therefore not easily scalable to more complex settings.

The approach we will advocate in this paper is however conceptually quite different.
{\em Rather than minimizing the equation errors, we look for the least complex model reconstructing the observations}. 
What we mean by 'least complex model' is somehow up to the user to decide. 
In this paper we consider a specific complexity measure defined as follows which will eventually reduce the 
inference problem to an optimization problem which can be solved efficiently, and for which we prove consistency.
\begin{Definition}[Lipshitz Condition]
	Consider a function$f:\R\rightarrow\R$.
	Assume there exists a constant $L$ such that one has for all $z,z'\in\R^d$ that 
	\begin{equation}
		\left|f(z) - f(z')\right| \leq L \left| z-z'\right|,
		\label{eq.lipschitz}
	\end{equation}
	then $f$ is Lipschitz smooth with a constant $L$.
	\label{def.lip}
\end{Definition}
Now, we structure the model class by the following nested sets
\begin{equation}
	\mathcal{F}_L = \left\{ (f,a)\in\mathcal{F} \ \Big| \ f: \mbox{Lipschitz with constant } L, \|a\|_2=1\right\}.
	\label{eq.classL}
\end{equation}

Now if $L_1<L_2<\dots < L_k$ are $k$ sorted constants, then one has a nested structure 
over the model class, i.e.
\begin{equation}
	\mathcal{F}_{L_1} \subseteq \mathcal{F}_{L_2} \subseteq \dots 
	\subseteq \mathcal{F}_{L_k} \subseteq \mathcal{F}.
	\label{eq.SR}
\end{equation}
A plausible identification algorithm would now be {\em to find the linear parameters $a\in\R^d$ such that  the mapping from $\{a^T\U_t\}_t$ to the corresponding values $\{y_t\}_t$ has as small a Lipschitz value 
$L$ as possible}. Since we are only interested at this stage in the parameters $a$ rather than also recovering $f$,  we focus attention on the given samples only. 
Consequently, sufficient condition for our needs for a function $f$ to be  Lipschitz smooth with constant $L$ is 
\begin{equation}
	\left|f(a^T\U_i) - f(a^T\U_j)\right| \leq L \left| a^T(\U_i -\U_j)\right|, \ \forall i<j=d,\dots,T.
	\label{eq.lipschitz2}
\end{equation}
By exploiting the monotonicity property of $f$, 
one can write the $O(T^2)$ constraints equivalently using only $O(T)$ constraints by ordering the data.
This step captures the function $f$ implicitly (see Figure (\ref{fig.mon})), proven as follows. 
\begin{Lemma}[Existence of a Transformation Function]
	Given a collection of pairs $\{(z_{(i)},y_{(i)})\}_{i=1}^n$, 
	enumerated such that $y_{(i)}\leq y_{(j)}$ if and only if $i\leq j$.
	Then we consider the sample conditions for $L<\infty$:
	\begin{equation}
			0\leq (y_{(j)}-y_{(i)})\leq L\left(z_{(j)}-z_{(i)}\right) \ \ \forall i<j=1,\dots,n,
			\label{eq.Ln}
		\end{equation}
		
	\begin{enumerate}
	\item If one has for a finite  $L>0$ that  (\ref{eq.Ln}) holds,
		then there exist a monotonically increasing function $f:\R\rightarrow\R$ 
		with Lipschitz constant $L$ interpolating the samples.
		
	\item If one has that for all admissible $(z,y)\in\R\times\R$ one has 
		that $y = f(z)$ for an (unknown) 
		continuous, (finite) differentiable and monotonically increasing 
		function $f:\R\rightarrow\R$, 
		then there is an $L<\infty$ such that 
		(\ref{eq.Ln}) holds.
	\end{enumerate}
\end{Lemma}
\begin{proof}
	To proof 1, consider the linear interpolation function $f_n:\R\rightarrow\R$, defined as 
	\begin{equation}
		f_n(z) = \frac{z - z_{\underline{z}(z)}}{z_{\overline{z}(z)} - z_{\underline{z}(z)}}
		\left(y_{\overline{z}(z)}-y_{\underline{z}(z)}\right) + y_{\underline{z}(z)},
		\label{eq.hn}
	\end{equation}
	where we define 
	$\overline{z}(z) = \arg\min_i(z_i: z_i\geq z)$ 
	and $\underline{z}(z) = \arg\max_i(z_i: z_i\leq z)$.
	Direct manipulation shows that this function is monotonically increasing and continuous.
	Now take $z<z'\in\R$, then we have to show that $f_n(z') - f_n(z) \leq L (z'-z)$.
	For convenience of notation define
	$l=\underline{z}(z)$, $u=\overline{z}(z)$, $l'=\underline{z}(z')$ and $u'=\overline{z}(z')$, then 
	\begin{multline}
		\frac{z'- z_{l'}}{z_{u'}-z_{l'}} (y_{u'}-y_{l'}) - \frac{z- z_l}{z_u-z_l} (y_u-y_l) + (y_{l'} - y_{l})\\
		\leq L(z' - z_{l'}) - L(z-z_l) + L(z_{l'}-z_l)\\
		= L(z'-z),
		\label{eq.hn}
	\end{multline}
	since $z_l\leq z_{l'}$ and $y_l\leq y_{l'}$ by definition.

	Item 2 is proven as follows.
	Let $f'$ be the derivative of a differentiable function $f_0$, 
	then the mean value theorem asserts that for any two samples 
	$(z_i,y_i)$ and $(z_j,y_j)$ for which	$z_i\leq z_j$,
	there exists a $z\in(z_i,z_j)\subset\R$ such that 
	\begin{equation}
		(y_j-y_i)
		= (z_j-z_i) f'(z) \leq L (z_i-z_j)
		\label{eq.mv}
	\end{equation}
	where $L = \sup_z f'(z)$.
\end{proof}
\begin{figure}[htbp] 
    \begin{center}\includegraphics[width=2.5in]{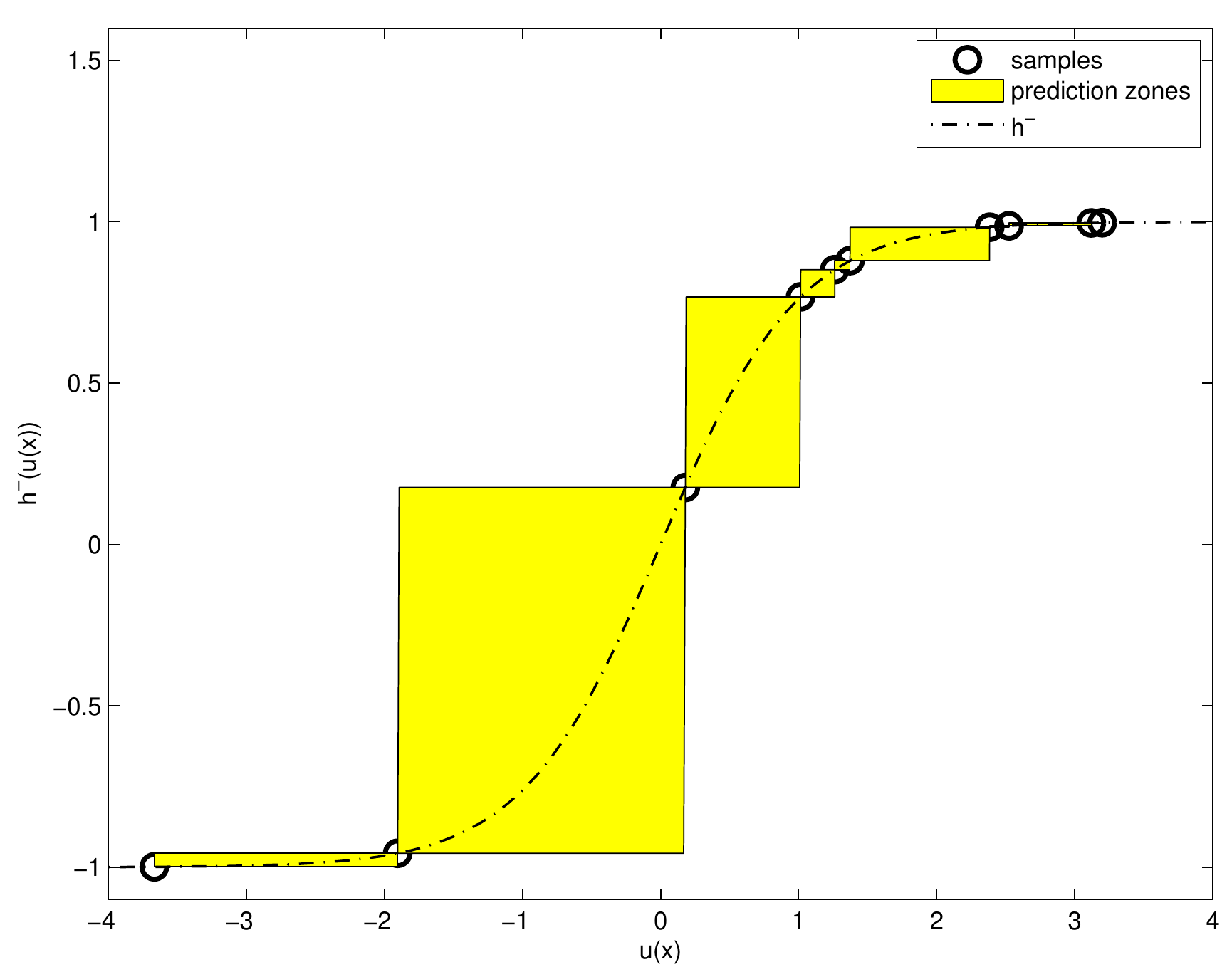}\end{center}
   \caption{\em
	Schematic representation of Lemma 1.
	If a Lipschitz-smooth monotone $f$ exists (black dash-dotted curved line), 
	samples obey the pairwise Lipschitz constraint.
	If samples exist satisfying the Lipschitz constraints, one can always 
	find a monotone function interpolating this samples (indicated by the yellow blocks).
   }
   \label{fig.mon}
\end{figure}
This observation motivates the following procedure: 
find parameters $a$ such that $f$ has minimal Lipschitz condition.
The solution is given by solving
\begin{multline}
	\max_{a}\min_{y_i\neq y_j} \frac{|a^T(\U_i-\U_j)|}{|y_i-y_j|} \\
	\mbox{\ \ s.t. \ \ } a^T\U_{(i)} \geq a^T\U_{(i-1)}, \ \ \forall i=d+1,\dots,T,
	\label{eq.minlip1}
\end{multline}
or
\begin{multline}
	\min_{a,L} \ L^2 \
	\mbox{\ \ s.t. \ \ } \|a\|_2=1, \\ 
	(y_{(i)}-y_{(i-1)}) \leq L \left(a^T(\U_{(i)}-\U_{(i-1)})\right), \ \ \forall i=d+1,\dots,T.
	\label{eq.minlip2}
\end{multline}
where we have only a linear number of constraints.After a change of variable, we can write equivalently
\begin{Definition}[MINLIP]
	Given an ordered set of examples $\{(\U_{(i)},y_{(i)})\}_{i=d}^T \subset \R^d\times\R$ indexed such that 
	$y_{(i-1)}\leq y_{(i)}$ for all $i=d+1,\dots,T$, then our (rescaled) estimate $a_T$ follows by solving
	\begin{multline}
		a_T = \argmin_{a} \ a^Ta
		\mbox{\ \ s.t. \ \ } \\ (y_{(i)}-y_{(i-1)}) \leq  a^T(\U_{(i)}-\U_{(i-1)}), \ \forall i=d+1,\dots,T.
		\label{eq.minlip}
	\end{multline}
	where the estimated function $\hat{f}$ is specified only implicitly as in Proposition 1.
\end{Definition}
This problem can be cast as a convex Quadratic Program (QP) with $T-d$ linear constraints
and $d$ unknowns.
This problem can be solved efficiently with contemporarily solvers available in most 
mathematical packages\footnote{In our experiments we use the solver available at http://www.mosek.org.}. 
The $a_T$ which minimizes this constrained objective is our estimate of a 
(rescaled) version of the parameters of the FIR system $H(q)$. The problem is written in matrix notation as 
\begin{equation}
	a_T = \argmin_{a\in\R^d} a^Ta \mbox{\ s.t. \ } (\Delta\U) a  \leq \Delta y, 
	\label{eq.minlip2}
\end{equation}
where $\U=(\U_{(d)}, \dots, \U_{(T)})\in\R^{(T-d+1)\times d}$ is a Hankel matrix (up to the sorting!), 
$y=(y_{(d)}, \dots, y_{(T)})\in\R^{T-d+1}$ is an ordered vector and 
\begin{equation}
	\Delta = \bmat{-1&1 &0 && 0 \\
			0 &-1&1 && \\
			&&\ddots&\ddots &\\
			0 &&& -1&1}\in\{-1,0,1\}^{(T-d)\times(T-d+1)}. 
	\label{eq.minlip2}
\end{equation}
In order to improve reproducibility of the result and stress practical use,
a full MATLAB implementation in 8 lines of code is given in Alg. (\ref{alg.minlip}).
In order to extend this implementation to handle general cases
one should take additional care of (possible) ties of output samples 
(See Subsection IV.C for a practical way to cope with this issue).

\begin{algorithm}
\caption{A MATLAB implementation of MINLIP}
\label{alg.minlip}
{\small\begin{verbatim}
function a = MINLIP(u,y,d)
n  = length(y);
x1 = toeplitz(u(d:end),u(d:-1:1)); 
y1 = y(d:end);
[ys,si]=sort(y1); 
xs = x1(si,:);
e  = ones(n,1);
D  = full(spdiags([-e e],[0 1],n-d,n-d+1));
a  = quadprog(eye(d),zeros(d,1),-D*xs,-D*ys);
\end{verbatim}}
\end{algorithm}

In a second phase,
it could be useful to reconstruct $f_0$ based on $a_T$ and the samples $\{(a_T^T\U_t,y_t)\}_{t=d}^T$.
We suggest to use the linear interpolation defined in (\ref{eq.hn}) to proof Lemma 1.
This function is by construction Lipschitz smooth with constant $L = \sqrt{a_T^Ta_T}$.
It is not too difficult to come up with more parsimonious estimators of the univariate function $f_0$ based 
on the bivariate samples $\{(a_T^T\U_t,y_t)\}_{t=d}^T$ which behaves more robust against modeling errors.

\subsection{Almost Consistency of MINLIP in the Noiseless Case}

This section characterizes how well the estimate approach the true impulse response,
under suitable assumptions on the data and the static nonlinearity.
Particularly, we assume that the observations arise from a {\em true} Monotone Wiener model with 
FIR system of given order for the dynamic part, and that no noise perturbs the observations.
This problem is non-trivial as 
(i) the proposed model is essentially nonlinear and non-parametric, and 
(ii) the method is not based directly on minimizing a mismatch between the model and the observations. 
The main outcome is that (approximate) consistent estimates
are given when the data satisfies a condition of local Persistency of Excitation (PE), and 
the true monotone static nonlinearity is smooth around its steepest part.

This result is referred to here as {\em almost consistency}, 
as it guarantees accuracy of the estimates only up to a small (but often non-zero) approximation term.
In a sense, this is the best one could hope for here because of two reasons,:
(a) the model is non-parametric (or semi-parametric) as the static monotone function 
of the model cannot be expressed straightforwardly in terms of a (small number of) parameters.
In that respect, a finite dataset contains never have enough information in order to reconstruct this system exactly.
(b) A finite dataset can never be locally exciting in every (arbitrary small) neighborhood, but can only guarantee 
this condition for all localities which are sufficiently large.
Classical concepts as bias or variance do not cover this notion as no stochastic assumptions are made.
The question wether approximate consistency implies asymptotic consistency requires one as well 
to make additional stochastic assumptions underlying the data, and is not covered in this text as such.
The analogue for linear estimating of a FIR model goes as follows:
assume the system can be described exactly as a FIR model of given order (smaller than) $d$,
then the corresponding notion is that the least squares estimate is {\em exactly} consistent
if the data is (globally) PE to an order $d$. Since we assume that there is no noise in the data,
there is no approximation to be made here. 
This difference can be seen as the consequence of the semi-parametric model 
of the Monotone Wiener system where in general no finite/small parameterization exists matching the system.

Assume that the observed system obeys the relation given in (\ref{eq.wiener.fir}) with a fixed (but unknown) 
monotonically increasing function $f_0:\R\rightarrow\R$ which is Lipschitz monotone with constant $L_0<\infty$, 
and parameter vector $a_0\in\R^d$. We refer to those as to the {\em true function $f_0$} and the {\em true parameters $a_0$} 
respectively. We address the question wether if we see enough data (or $T\rightarrow\infty$), 
the MINLIP estimate $a_T$ will equal $a_0$ up to a scaling constant.
Formally, we consider the MINLIP estimator based on the (infinite) set $\Obs$ as 
\begin{multline}
	\hat{L} = \min_{L,\|a\|_2=1} L \\ 
	\mbox{ \ s.t. \ } (y-y') \leq L a^T(\U-\U'), \ \forall (\U,y),(\U',y')\in\Obs, y>y'. 
	\label{eq.minlip5}
\end{multline}
Sometimes it will be convenient to rewrite the MINLIP estimator (\ref{eq.minlip5})
as the following minimax problem:
\begin{equation}
	\ell = \max_{\|a\|_2=1} \inf_{(\U,y),(\U',y)\in\Obs: y>y'} \frac{a^T(\U-\U')}{y-y'}.
	\label{eq.minlip6}
\end{equation}
where $\ell\geq \frac{1}{L_0}$ by construction of $L_0$.
In order to characterize the solution, the following two conditions are needed. 
\begin{Definition}[$f$ is $(L_0,g)$-Lipschitz on $\Obs'\subseteq\R$]
	The function $f$ is said to be $(L_0,g)$-Lipschitz on $\Obs'\subseteq\R$ 
	for a decreasing, positive function $g:\R^+\rightarrow\R^+$ with $g(0)=1$ if:
	(A) one has for all $z,z'\in\Obs'$ that 
	\begin{equation}
		\left(f(z) - f(z')\right) 
		\leq L_0\left(z-z'\right).
		\label{eq.lip0}
	\end{equation}	
	(B) there exists a $z\in\Obs'$ and $z'\in\Obs'$ such that 
	\begin{equation}
		\left(f(z) - f(z')\right) 
		= L_0\left(z-z'\right), 
		\label{eq.lip1}
	\end{equation}
	and (C) one has for this $z$, for any $\epsilon>0$ and 
	$z''\in\Obs'$ where $|z-z''|\leq \epsilon$ that 
	\begin{equation}
		\left|f(z) - f(z'')\right| \geq g(|z-z''|) L_0 |z-z''| .
		\label{eq.lip2}
	\end{equation}
\end{Definition}
Hence $g$ denotes how 'smooth' the constant $L$ decays in a neighborhood of $z$ 
where the actual Lipschitz constraint is met (that is, a slower decaying function $g$ indicates a higher smoothness).
In particular, a value $h(\epsilon)=1$ implies that the function $f$ is linear with slope $L_0$
in this neighborhood. 
Such characterization is illustrated for $f(z)=\tanh(z)$ in Figure (\ref{fig.lip})
for a smoothness function $g(\epsilon) = 1/(1+c\epsilon)$.

\begin{Definition}[$\epsilon$-Local Persistently Exciting]
	We say that a set $\Obs\subseteq\R^d$
	is $\epsilon$-local persistent exciting of order $d$	
	for $\epsilon>0$
	iff for any vector $\U\in\Obs$,
	there	 exist $d$ vectors $\U_1,\dots,\U_d\in\Obs$
	with $\{\U-\U_k\}_{k=1}^m$ linearly independent vectors
	and
	\begin{equation}
		\|\U-\U_k\|_2\leq \epsilon, \ \forall k=1,\dots,d.
		\label{eq.localpe1}
	\end{equation}
\end{Definition}
This definition can be seen as a {\em local} version of Persistency of Excitation (PE),
see e.g. \cite{ljung87,soderstrom1989,goodwin1984} for the classical definition of PE.
\begin{figure}[htbp] 
    \begin{center}\includegraphics[width=2.5in]{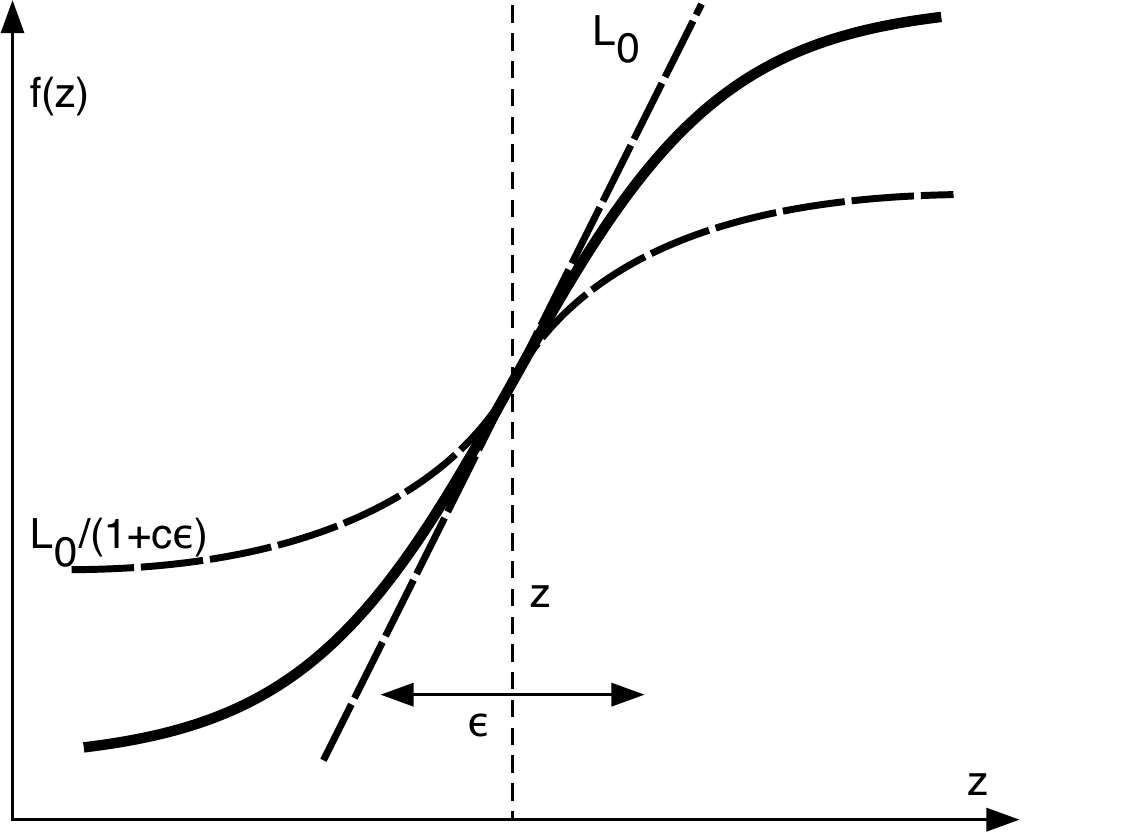}\end{center}
   \caption{\em
   	Schematic illustration of the $(L_0,g)$-Lipschitz property of a function $f$ 
	with $g(\epsilon)=\frac{1}{1+c\epsilon}$.
	There should be a sample $z$ where the Lipschitz constant $L_0$ is 
	attained, and in the $\epsilon$-neighborhood of this sample $z$
	the Lipschitz-property shouldn't decay too fast, e.g. in the neighborhood 
	of $z$ the function behave almost linearly.
   }
   \label{fig.lip}
\end{figure}

\begin{Theorem}[Almost Consistency]
	Fix $\epsilon>0$ and consider the $(f_0,a_0)$-Wiener system as in (\ref{eq.monowien})
	with corresponding observations in $\Obs$.
	If $f_0:\R\rightarrow\R$ is $(L_0,g)$-Lipschitz and monotone on the set 
	$\{(z,y): z=a_0^T\U: \U\in\Obs\}$ and $\Obs$ is $\epsilon$-local PE, then
	\begin{equation}
		a_T^Ta_0 \geq g(\epsilon),
		\label{eq.consistent}
	\end{equation}
	where $a_T$ is the estimate of MINLIP as in (\ref{eq.minlip}).
\end{Theorem}
This means that the smoother the function $f_0$ is towards its steepest part, the better estimates we get.
In particular, when $f_0$ is (almost) linear - we only need global PE to have exact estimates $a_T\propto a_0$.
Specifically, if $g(\epsilon)=1$ - meaning that the function $h$ is linear - only global persistency of excitation 
is required to have consistency,  and the MINLIP estimator will return the same result as a linear estimator.

\begin{proof}
	Let the Lipschitz constant  be achieved in the sample $(\U,y),(\U',y')\in\Obs$, such that 
	\begin{equation}
		(y-y') =L_0 (\U-\U')^Ta_0.
		\label{eq.lmm2.1}
	\end{equation}
	As the set $\Obs$ is $\epsilon$-local persistent exciting of order $d$, 
	one can find for  the sample $(\U,y)\in\Obs$ $d$ vectors $\U_1,\dots,\U_d$ contained in
	$\Obs$ such that the vectors $(\U-\U_1), \dots, (\U-\U_d)$ are linearly independent 
	and have norm smaller than $\epsilon$.
	This implies that one can rewrite $a_0,a_T\in\R^d$ 
	(i.e. the true parameter vector and the optimal estimate associated to (\ref{eq.minlip5})) as 
	\begin{equation}
		\begin{cases}
			a_0 = \sum_{k=1}^d \alpha_{0,k} \overline{(\U-\U_k)} \\
			a_T = \sum_{k=1}^d \alpha_{k} \overline{(\U-\U_k)}, 
		\end{cases}
		\label{eq.lmm2.2}
	\end{equation}
	where $\alpha,\alpha_0\in\R^d$, 
	and we let $(\U-\U_k) = \sigma_k\overline{(\U-\U_k)}\|(\U-\U_k)\|_2$ with $\sigma_k\in\{-1,1\}$ such that 
	$\|\overline{(\U-\U_k)}\|_2=1$ and $\overline{(\U-\U_k)}^T a_0\geq 0$ for all $k=1,\dots,d$. 
	Define as previously for each $k=1,\dots,d$ the constant $L_k\in\R_+$ such that 
	$(y-y_k) = L_k (\U-\U_k)^Ta_0$, where $L_k\leq L_0$ by construction.
	Now define the matrix $D_\U\in\R^{d \times d}$ as 
	\begin{equation}
		D_\U = \bmat{\overline{(\U-\U_1)} \\ \overline{(\U-\U_2)} \\ \vdots \\ \overline{(\U-\U_d)} },
		\label{eq.lmm2.3}
	\end{equation}
	such that $D_\U a_0\geq 0_d$, 
	and the matrix $L\in\R^{d\times d}$ as  $L = \diag(L_1,\dots, L_d)$.
	Then we have in matrix notation that $a_0 = D_\U^T\alpha_0$ and $a_T = D_\U^T\alpha$.
	By construction of the MINLIP estimator we have that 
	\begin{equation}
			LD_\U D_\U^T \alpha_0 = LD_\U a_0 \leq \ell D_\U a_T = \ell D_\U D_\U^T \alpha,
			\label{eq.lmm2.3b}
	\end{equation} 
	where $\ell$ is the minimal value obtained in (\ref{eq.minlip5}).
	As such $\frac{1}{\ell}LD_\U D_\U^T \alpha_0 \leq D_\U D_\U^T \alpha$.
	Then one has 
	\begin{multline}
		a_0^Ta_T = \alpha_0^T(D_\U D_\U^T) \alpha
		\geq  \frac{1}{\ell} \alpha_0^T  L(D_\U D_\U^T) \alpha_0\\
		\geq \frac{g(\epsilon)L_0}{\ell} \alpha_0^T(D_\U D_\U^T) \alpha_0 \\
		\geq g(\epsilon) \alpha_0^T(D_\U D_\U^T) \alpha_0 
		\geq g(\epsilon),
		\label{eq.lmm2.4}
	\end{multline}
	since $\ell\leq L_0$.	This proofs the result.
\end{proof}

\section{IDENTIFICATION WITH NOISY DATA}

This section considers how to modify the estimator towards the case
of noise being present in the data, or where the data-samples are
only approximated by a monotone FIR $(f,a)$-system.
Empirical evidence is provided for the use of this estimator.

\subsection{Model Class}

There are a number of different ways one can model noise in the class of monotone wiener systems.
A first one is to consider noise on the measured outputs (or {\em measurement noise}). One may argue that this 
model is not a very realistic assumption in case the observations are a quantized version of the output of the linear system. 
That is, once the signal is quantized (and transmitted), it can often be measured without error.
On the other hand,  it is often not clear which noise model of a quantized signal (with a finite number of different levels)
fits the application (a Gaussian distribution would not make much sense here). 
Another assumption one can make is that noise occurs in the signal $\{u_t\}_t$, 
but as this results in coloring of the noise by the unknown linear subsystem, this model is not adopted as yet.
A third alternative is that the noise comes in between the  linear subsystem and the monotone static 
function (see Fig. \ref{fig.ewiener}). As in the following no restrictive assumptions (as whiteness of the noise signal)
is assumed, this could be seen as uncertainty coming in in the model by under-modeling of the linear system.
This is the view which underlies the following definitions.

\begin{Definition}[Noisy FIR Wiener Model $(f,a)$]
  A FIR Wiener model consists of a sequence of 
  (i) a linear dynamical model characterized by an impulse response function $H(q^{-1})$ 
  applied on the {\em input} signal $\{u_t\}_t$,
  (ii) a static nonlinear function $f:\R\rightarrow\R$, and 
  (iii) a sequence of 'noise' terms $\{e_t\}_t$.
  If the signals $\{u_t\}_t$, $\{y_t\}_t$ and $\{e_t\}_t$ follow such a model with 
  'true' subsystem $H_0$ and 'true' function $f_0$, we can write 
  \begin{equation}
  	y_t = f_0\Big( H_0(q^{-1})(u_t) + e_t \Big),
  	\label{eq.wiener}
  \end{equation}
  If the signals $\{u_t\}_t$, $\{y_t\}_t$ and $\{e_t\}_t$ obey a Wiener FIR-model 
  with 'true' function $f_0$ and 'true' parameters $a_0$, or 
  \begin{equation}
  		y_t 	= f_0\left( \sum_{k=1}^{d} a_{0,k} u_{t-k} + e_t\right)
    			= f_0( a_0^T \U_t +e_t),
    \label{eq.wiener.fir}
  \end{equation}
  where $a_0\in\R^d$ and $\|a_0\|_2=1$.
\end{Definition}

\begin{figure}
    \includegraphics[width=2.5in]{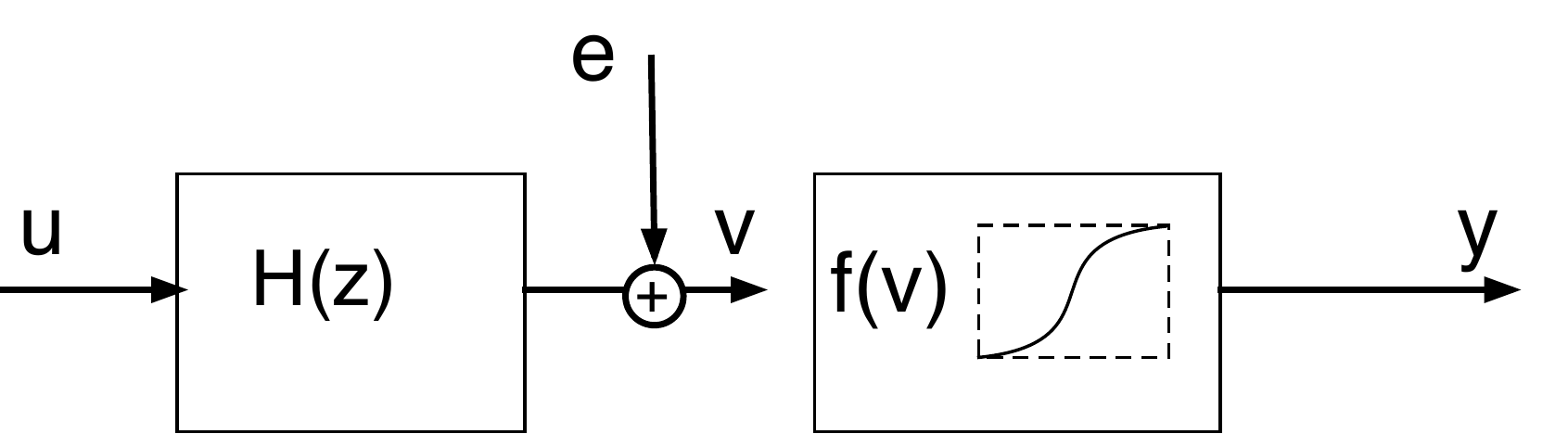}
  \caption{\sl  Schematic representation of a noisy monotone  wiener system.
  	This paper adopts the setting that noise comes in 
	after the linear dynamic part (capturing model mismatch), and 
	right before application of the static nonlinearity (or quantization).
  }
  \label{fig.ewiener}
\end{figure}

\subsection{MINLIP for Noisy Data}

Given time-series $\{u_t\}_t$ and $\{y_t\}_t$, referred to as 'input'  and 'output'.
Again, let $\{(\U_t=(u_{t-d+1},\dots, u_{t}),y_t)\}_{t=d}^T\subset\R^d\times\R$ 
be a dataset containing $T-d+1$ samples.
Let this set be reindexed as  
$\{(\U_{(j)},y_{(j)})\}_{j=d}^T$ where $y_{(i)}\leq y_{(j)}$ for all $d<i<j\leq T$.
Then adopting the noisy model (\ref{eq.wiener.fir}) suggests 
modification of the standard MINLIP (see eq. (\ref{eq.minlip})) as 
\begin{multline}
	\min_{a,e} \ \frac{1}{2}a^Ta + \frac{\gamma}{2} \sum_{t=d}^T |e_i| \\
	\mbox{\ \ s.t. \ \ }  (y_{(i)}-y_{(i-1)}) \leq 
	(a^T\U_{(i)}+e_{(i)}) - (a^T\U_{(i-1)} + e_{(i-1)}), \\ 
	\ \forall i=d+1,\dots,T,
	\label{eq.nminlip}
\end{multline}
where the fixed {\em regularization parameter} $\gamma>0$
trades the Lipschitz based regularization term and the penalization of the residuals.
The choice of penalizing the absolute loss of the residuals is inspired by 
(i) robustness considerations and the (ii) non-stochastic nature of the residuals where a worst-case approach is more suited.
The tuning of this constant can be done with an appropriate model selection 
criterion as cross-validation.
As before, this optimization problem can be solved efficiently as a convex 
quadratic program (QP) using $O(T)$ unknowns and inequality constraints.

\section{Empirical Evidence}

This section spells out  a number of artificial yet challenging case studies.
A main argument which is made here is that although the underlying system $H_0$
may be represented as a fractional polynomial, 
it is often useful to consider a FIR model consisting of a large number of tapped delays (say 
$d=O(100)$)  which can approach (the impulse response function of) 
$H_0$ arbitrarily close. In this way, one does not have to specify 
explicitly model order or delay of the model. We will refer to this approach as an over-parametrization.
In order to choose the number $d$ of tapped delays in the FIR model, 
one may perform a nonparametric analysis of the impulse response of the data (neglecting nonlinear effects as yet).
It is found that MINLIP is especially appropriate for such an over-parametrization approach and doesn't loose much efficiency 
as it builds in explicitly a mechanism of model complexity control and regularization,  
dealing with ill-posedness problems often present in such a context.

\subsection{The Experimental setup}

The Monotone Wiener systems from which the data is generated take the following form.
The linear subsystem $H_0(q^{-1})$ are represented as fractional polynomial models as 
\begin{equation}
	H_0(q^{-1}) = \frac{B(q^{-1})}{A(q^{-1})} 
	= \frac{b_0 + b_1 q^{-1} + \dots + b_{2m_z}q^{-{2m_z}}}{1 + a_1 q^{-1} + \dots + a_{2m_z} q^{-{2m_p}}},
	\label{eq.H1}
\end{equation}
where $2m_z>0$ and $2m_p>0$ denote the orders of the polynomials $A(q^{-1})$ and $B(q^{-1})$.
Those polynomials are chosen such that they have $m_z$ and $m_p$ conjugate pairs of zeros and poles respectively.
The zeros of $A(q^{-1})$ are referred to as poles of $H_0(q^{-1})$, 
and the zeros of $B(q^{-1})$ are referred to as zeros of $H_0(q^{-1})$.
In this example, we set $n_z = 2$ and $n_p=20$.
The conjugate poles and conjugate zeros are uniformly at random picked inside the unit circle
(see Figure \ref{fig.ex1} for an example).
In general, we see that a FIR representation of $d=200$ is sufficient to capture the dynamics of such a system.
The output nonlinearity is fixed as $f_0:\R\rightarrow\R$ where for $x\in\R$ one has
\begin{equation}
	f_0(x) = 2+ \tanh(5x+2) + 0.5\tanh(5x-3).
	\label{eq.f0}
\end{equation}
This function is somewhat challenging as it cannot be described as a simple saturation function, 
is not symmetric around any point, and has an almost zero gradient in $x=0$.
Then a monotone Winer system is constructed as
\begin{equation}
	y_t = f_0(g H_0(q^{-1})u_t), \ \forall t=1,\dots,T,
	\label{eq.h1f}
\end{equation}
where the gain $g>0$ is chosen such that the values $\{gH_0(q^{-1})u_t\}_t$ have a unit standard deviation.
The estimates of MINLIP on a time-series of length T=450, 500,550,600 - taken from 
the Wiener System of Fig. (\ref{fig.ex1}), Fig. (\ref{fig.ex2}) - 
is displayed in Fig. (\ref{fig.ex2}). 
Here a FIR approximation of $d=200$ is used, capturing the dynamics of the system 
$H_0$ reasonably well (see Fig. (\ref{fig.ex1}.a)).

\begin{figure}[htbp] 
	\begin{tabular}{lr}
	\subfigure{\includegraphics[width=1.6in]{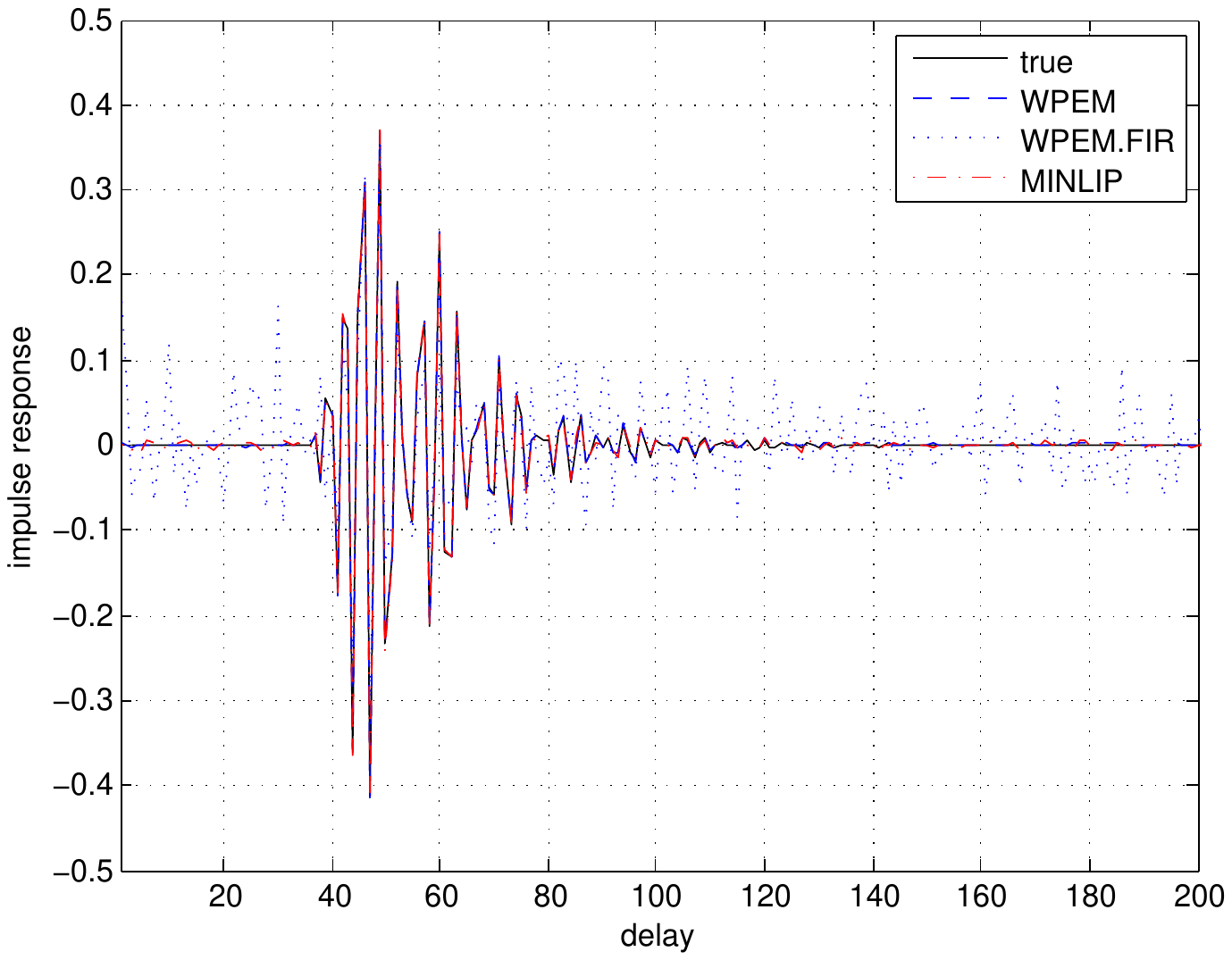}} &
	\subfigure{\includegraphics[width=1.6in]{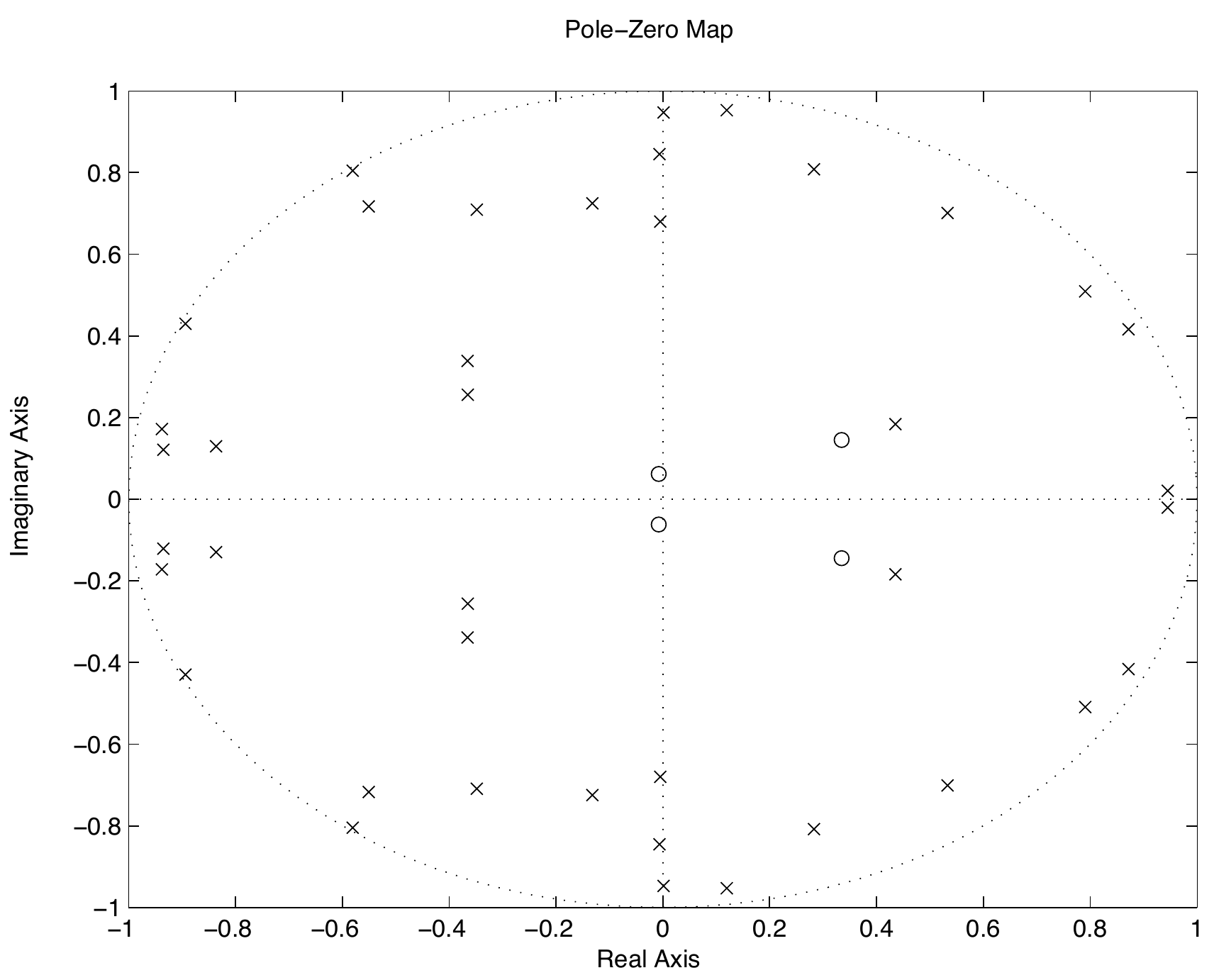}} 
	\end{tabular}
   \caption{\em
	An example of a system $H_0$ randomly generated, with $n_p=20$ and $n_z=2$.
	Panel (a) shows the resulting impulse response (and hence a FIR approximation) up to lag $d=200$.
	Panel (b) displays the conjugate poles and conjugate zeros in the complex domain 
	using a pole-zero plot of the system.
	Observe (i) the presence of a considerable (non-zero) delay, and 
	(ii) the fact that some poles are located close to the unit circle.
	This makes a inverse modeling approach unfeasible.
   }
   \label{fig.ex1}
\end{figure}

\begin{figure}[htbp] 
	\begin{tabular}{lr}
	\subfigure{\includegraphics[width=1.3in]{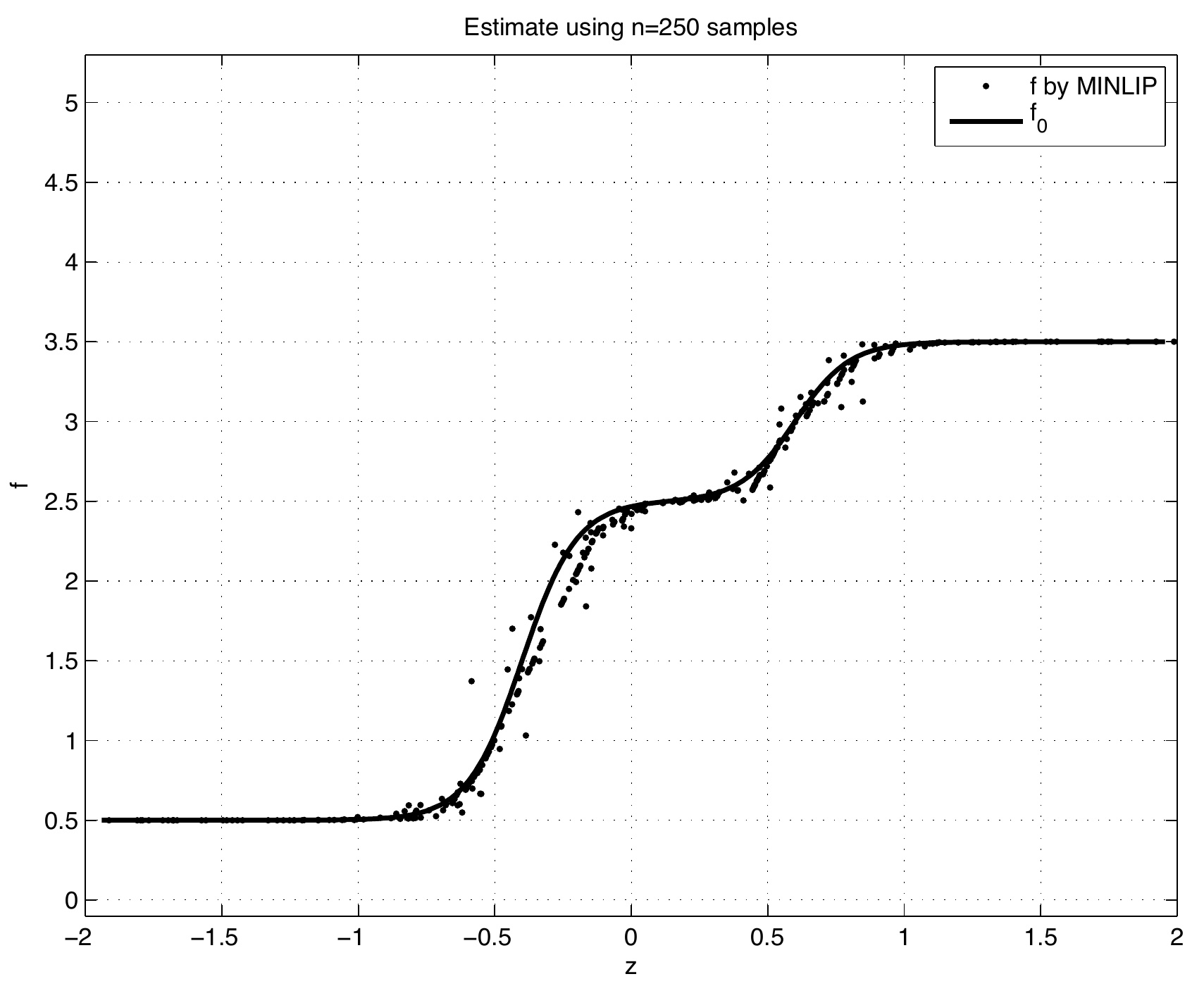}} &
	\subfigure{\includegraphics[width=1.3in]{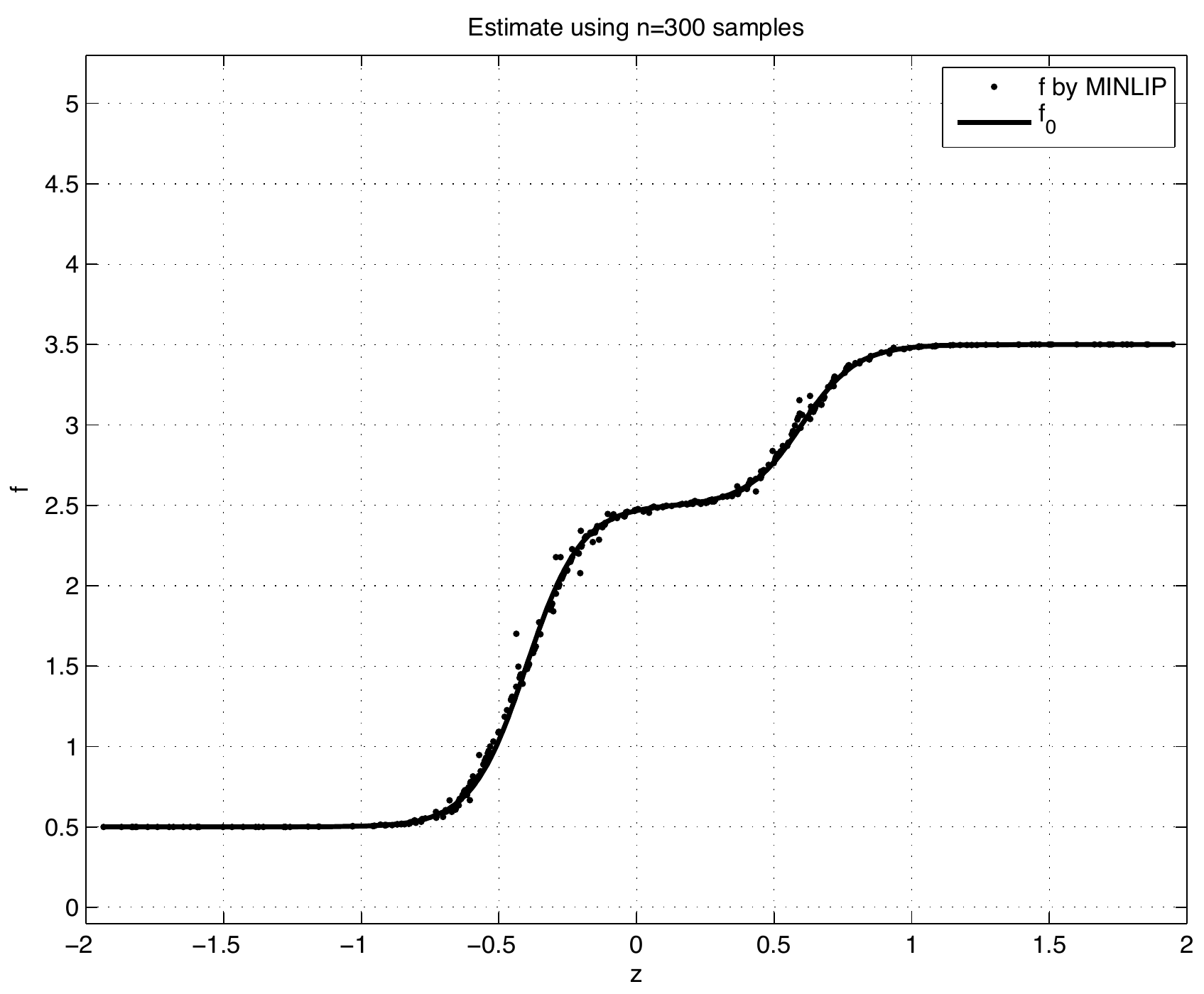}} \\
	\subfigure{\includegraphics[width=1.3in]{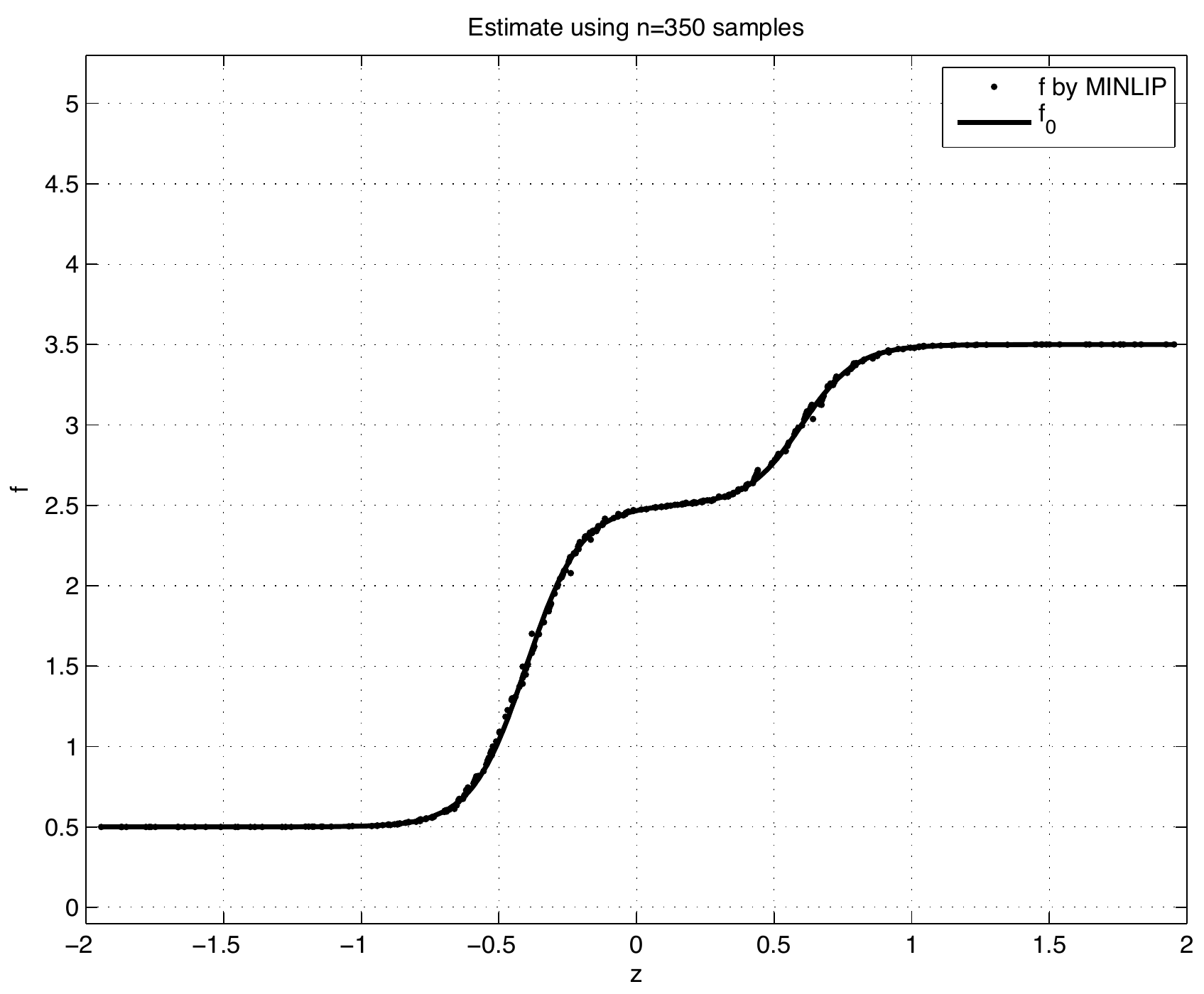}} &
	\subfigure{\includegraphics[width=1.3in]{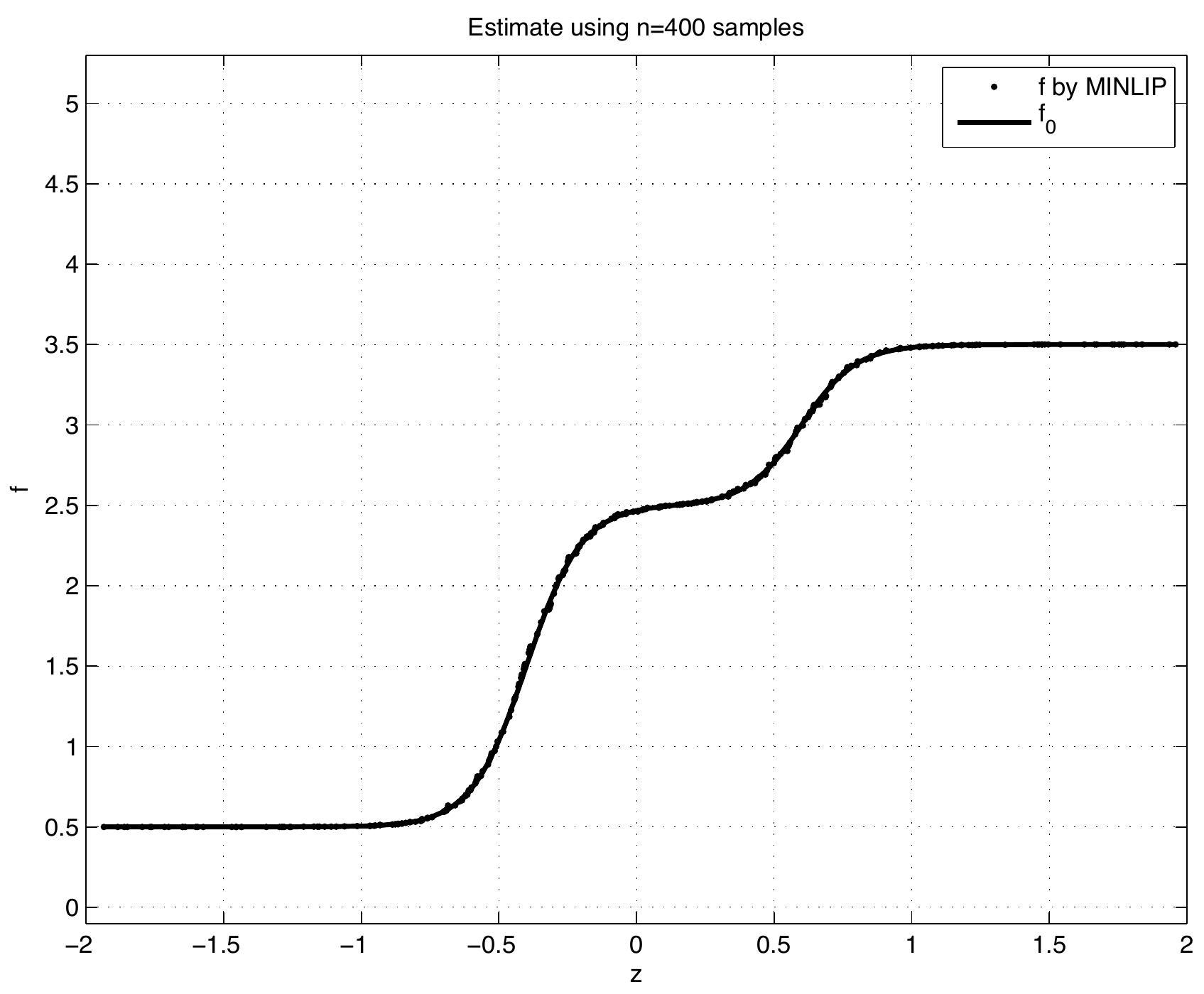}} 
	\end{tabular}
   \caption{\em
	Evolution of the estimate of MINLIP when provided by signals of length 
	(a) 450, (b) 500, (c) 550, (d) 600 samples, taken from the Wiener model described in Figure (\ref{fig.ex1})
	and eq. (\ref{eq.f0}). A FIR model of $d=200$ is used to approximate the linear system, taking
	care of the delay as well of the (unknown) model orders.
   }
   \label{fig.ex2}
\end{figure}

The following 6 different identification  methods were implemented to benchmark the MINLIP against:
\begin{enumerate}
\item{(LS '$x-y$')} In order to provide a (naive) lower-bound to the performance, 
	a FIR identification 
	technique based on a Least Squares (LS) argument 
	was implemented on the signals $\{u_t\}_t$ and $\{y_t\}_t$ directly, 
	neglecting the Wiener structure altogether.

\item{(LS '$x-z$')}
	In order to get an upper-bound on the performance of the identification technique, 
	an ARX identification technique was implemented based on the 
	(latent) intermediate signal $\{z_t = H_0(q^{-1})u_t\}_t$.

\item{(WPEM FIR)}
	We consider a FIR model structure of sufficiently high order (here $d=200$) 
	such that (\ref{eq.H1}) can be represented fairly well, 
	and we let the corresponding FIR coefficients act directly as unknowns.
	The nonlinearity is represented as a piecewise function
	based on 20 fixed knots which were optimally tuned to the example at hand. 
	Global optimization on both FIR coefficients as well as on the unknowns of the nonlinearity 
	is performed by the Broyden-Fletcher-Goldfarb-Shannon (BFGS) method implemented in MATLAB
	in the \verb|fminunc| function.

\item{(WPEM ARX)}
	Here we implement the same approach now based on a class of ARX models 
	representing the optimal predictors corresponding to   (\ref{eq.H1}). 
	Now, we let the poles and zeros of this model class act as unknowns directly, and 
	As before, the nonlinearity is expressed as a piecewise linear function with fixed grid points.
	Global optimization is performed by the BFGS method.

\item{(Greblicki2002)}
	The smoothing approach as described in \cite{greblicki2008}
	is implemented as well. This method works directly on a FIR overparameterization 
	of the (linear sub-) system, and gives reasonable estimates when sufficiently many input-samples
	following a stochastic (approximately white) Gaussian process are provided.

\item{(Bai2006)}
	The last approach we benchmark against is the technique 
	described in \cite{zhang2006,bai2008} using prior knowledge of the monotonicity of the output function.
	This technique is implemented by solving 
	\begin{equation}
		 \min_a \sum_{j=d+1}^T \left( \sign(y_{(j)}-y_{(j-1)}) - \tilde\sign((\U_{(j)}-\U_{(j-1)})^Ta)\right)^2,
		\label{eq.bai}
	\end{equation}
	which is in our experiments solved by the BFGS method.
	We found that in order to make this approach to work well one needs to 
	resort to a smooth proxy '$\tilde\sign$' of the discrete function '$\sign(z) = I(z>0)-I(z<0)$', 
	making gradient information available at most unknowns in the search space. In many cases solving 
	this problem takes substantially more resources (CPU-power, memory) compared to the other techniques.
\end{enumerate}

Accuracy of an estimate is expressed in terms of the angle between the true impulse response of $H_0$ 
and the impulse response of the estimate $\hat{H}_T$. 
Let as classical the $L_2$ norm of a system $H$ be defined as 
\begin{equation}
	\sum_{t=0}^\infty (H_0(q^{-1})\delta_t)^2,
	\label{eq.L2}
\end{equation}
where $\delta_\tau$ is a time-series of all zeros except for the first location which equals one.
Then the correlation of two systems $H_0$ and $\hat{H}_T$ may be defined as 
\begin{equation}
	d(H_0,\hat{H}_T) = \frac{\sum_{t=0}^\infty H_0(q^{-1})\delta_t \ \hat{H}_T(q^{-1})\delta_t}{\|H_0\|_2 \|\hat{H}_T\|_2}.
	\label{eq.corr}
\end{equation}
If the systems $H_0$ and $\hat{H}_T$ have impulse response vector $h_0$ and $\hat{h}$ respectively,
this coefficient can be written as the Pearson correlation coefficient between those two vectors, or $\frac{h_0^T\hat{0}}{\|h_0\|_2 \|\hat{h}\|_2}$.
There are 3 reasons for adopting this definition.
The first is that the gain of the system $H_0$ is unidentifiable (hence cannot play a role in the quality measure), and 
the second one is that the impulse response is the common denominator for any LTI, independently of a parametrization. 
Thirdly, the current method concentrates on first instance only at identification of the linear system, while making predictions
requires additional estimation of the static nonlinearity. As such, measuring performance based on prediction accuracy would 
convolute the results with performance of this reconstruction step as well.

\subsection{A Noiseless Example}

The first experiment is based on noiseless data.
Figure (\ref{fig.perf1}) shows the results of the experiment, 
where in each iteration $T$ samples are generated from a random system $(f_0,H_0)$,
the different identification algorithms are carried out, and their respective accuracy is computed.
This iteration is performed 100 times for any $T = 300,400,500,\dots,1000$.
We see from the results that the MINLIP estimator converges fast to the 
best achievable performance (indicated by the LS '$x-z$' approach). 
The WPEM algorithms give in many cases unreliable results, performing much worse on the average.
Specifically, we find that it is bad practice to combine the WPEM on (overparametrized) 
FIR model, perhaps because global optimization often presents (numerical) problems 
when optimizing over so large a set of parameters.
MINLIP however can handle such overparametrization quite efficiently, and 
does as such not require to determine the model orders and the delay of the system.
This suggests the use of complexity control to give a principled tool to handle the task of 
model order selection.

\begin{figure}[htbp] 
	\includegraphics[width=3in]{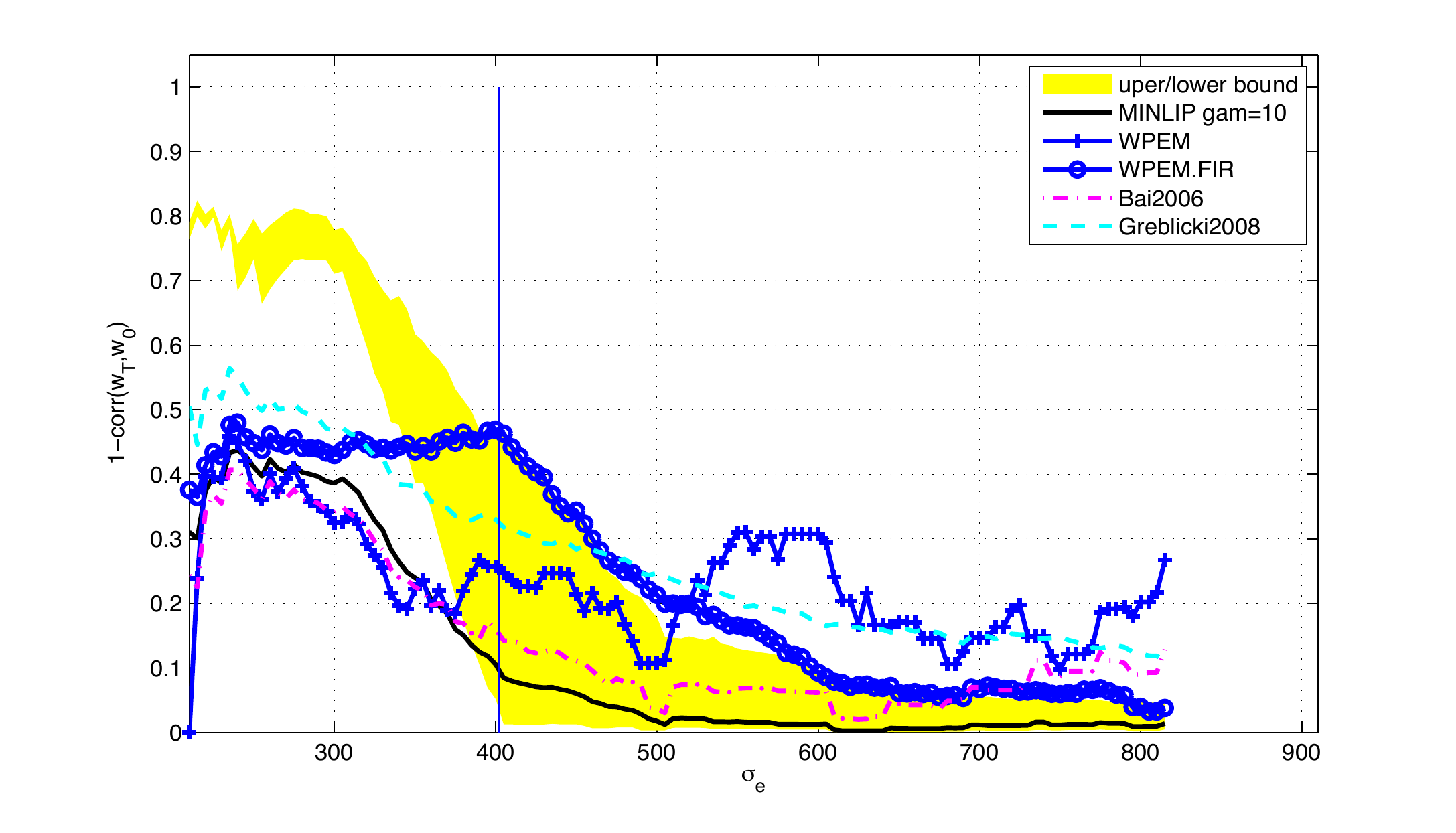}
   	\caption{\em
		Results of the first experiment using noiseless data generated from 
		a monotone Wiener nonlinearity as in eq. (\ref{eq.f0}) and systems $H_0$ as in (\ref{eq.H1}).
		Performance expressed as correlations of $H_0$ and $\hat{H}_T$ as in (\ref{eq.corr}) are displayed 
		for sample sizes ranging from T=210 to T=1000, using a FIR overparametrization of $d=200$.
		The vertical line denotes the place where a least squares technique would exactly reconstruct $H_0$
		if $z_t = H_0(q^{-1}) u_t$ were given.
	}
   \label{fig.perf1}
\end{figure}

\subsection{Quantized signals}

Here we investigate the use of MINLIP in case 
the output signal is quantized (i.e. takes a small number of different values).
This task poses additional challenges as there is a direct need to handle tied output values well.
We adapt the following procedure: if $y_{(i-1)}=y_{(i)}$ (tied),
then we cannot compare the corresponding values $a^T\U_{(i-1)}$ and $a^T\U_{(i)}$ unambiguously.
Rather, we compare $a^T\U_{(i-1)}$ and $a^T\U_{(i)}$ both with the samples $(y_j,a^T\U_j)$ and 
$(y_k,a^T\U_k)$ which have a strictly lower value $y_j<y_{(i-1)},y_{(i)}<y_k$.
Again by transitivity of the relation $<$ one can prune many of the relations in the final QP.
In the worst case only 2 different output levels are observed on $T$ samples, and both levels are each observed on $\frac{T}{2}$ samples.
Then one has to work with $\frac{T^2}{4}$ inequality constraints rather than the $O(T)$ ones in the standard implementation.
It may be argued that the theoretical account for Monotone Wiener systems as given above does not hold strictly as the 
steepest part ('the jumps') have an unbounded  Lipschitz constant. However, as the dataset is finite, this measure is 
necessarily finite and the method continuous to yield good solutions. If there are no samples present 'near the jumps', 
the analysis hold with an appropriate function $g$ dependent on this 'margin', and the size of the jumps.

In this example we define the output nonlinearity for any $z\in\R$ as  
\begin{equation}
	f'_0(z) = I(z>-0.5) + I(z>2),
	\label{eq.f1}
\end{equation}
with the output taking values in the set $\{0,1,2\}$, and the jumps of the function occurring when $z=-0.5$ and $z=2$.
Again, the linear systems $H_0$ used to generate the signals are as  in eq. (\ref{eq.H1}),
and we benchmark the MINLIP against the approaches described in the previous subsection (LS, WPEM and BAI).
The results are displayed in Fig. (\ref{fig.perf2}).

From these results we may suggest a few guidelines.
The first is that a naive LS '$x-y$' regression works surprisingly good,
and it is not at all trivial to beat this one. The reason the approaches based on 
global optimization (i.e. WPEM, WPEM.FIR and Bai2006) do not work as good might 
be that the discrete nature of the identification task translates in a highly non-smooth
cost surface given to the optimizer. 
This experiment however suggests that MINLIP achieves a solution which is 
often close to the best one could hope for (indicated by the LS '$x-z$' method).
The averaging approach proposed in \cite{greblicki2008} appears fairly robust to the quantization effects as well.

\begin{figure}[htbp] 
	\includegraphics[width=3in]{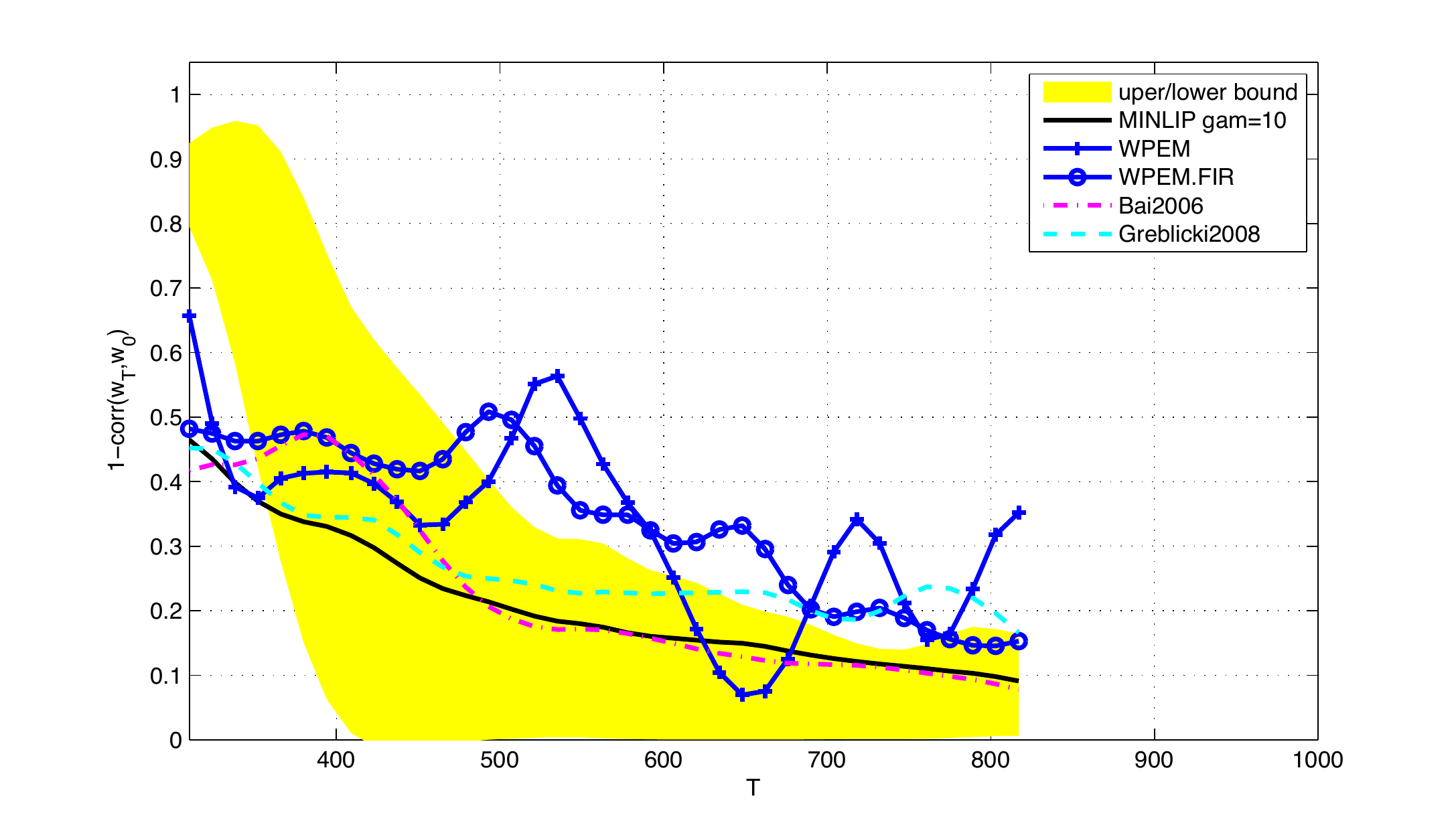}
   	\caption{\em
		Results of the quantization experiment using noiseless data generated from 
		a monotone Wiener nonlinearity as in eq. (\ref{eq.f1}) and systems $H_0$ as in (\ref{eq.H1}).
		Performance expressed as correlations of $H_0$ and $\hat{H}_T$ as in (\ref{eq.corr}) are displayed 
		for sample sizes ranging from $T=210$ to $T=1000$, using a FIR overparameterization of $d=200$.
	}
   \label{fig.perf2}
\end{figure}

\subsection{The noisy case}

This subsection reports results achieved with MINLIP in case the intermediate signal 
$\{z_t\}_t$ is perturbed by noise.
The experiment is set up as before in Subsection 2.3, but the system becomes now
\begin{equation}
	y_t = f_0\left(H_0(q^{-1}) u_t +e_t\right), \ \forall t=1,\dots,T,
	\label{eq.H2}
\end{equation}
where $f_0$ is as in eq. (\ref{eq.f0}) and $H_0$ is randomly generated as in (\ref{eq.H1}).
The terms $\{e_t\}_t$ are zero mean white Gaussian noise with standard deviation $\sigma_e>0$.
This experiment is conceived in a slightly different manner than before.
We consider a fixed number of samples $T=500$, and let the Signal-to-Noise Ratio (SNR) vary from 0.1 to 10
(or $\sigma_e=0.1,\dots,10$). 
As indicated in Section III, MINLIP is dependent on the choice of a suitable $\gamma>0$,
which is in turn depending on the noise variance $\sigma_e$.
As this characteristic is unknown in general in practical applications,
the choice of $\gamma$ is to be made based on a suitable model selection technique.
Actually, the problem of model selection in the context of a Wiener model is 
not covered as such, and prompts new questions related to information criteria, stability and consistency.
For now, we use a fixed value of $\gamma=10$ which works well in many cases.
The results are displayed in graph (\ref{fig.perf3}).
Those results indicate that the MINLIP outperforms the other techniques especially
when noise is small compared to the 'informative' signal, while all techniques become arbitrarily bad 
when this ratio grows.

In the next experiment fix an SNR of 3 and lets see what happens to the performance of the different estimators 
if the given signals have an increasing length. 
The evolution of the average accuracy of MINLIP and the competing estimators is given in Fig. (\ref{fig.perf4}).
From this result we see that the behavior is not too different from the noiseless case,
except the fact that the WPEM and WPEM.FIR approaches are not very robust to noise,
and the approach in (\ref{eq.bai}) is clearly a bad choice in this case and needs additional care.

\begin{figure}[htbp] 
	\includegraphics[width=3in]{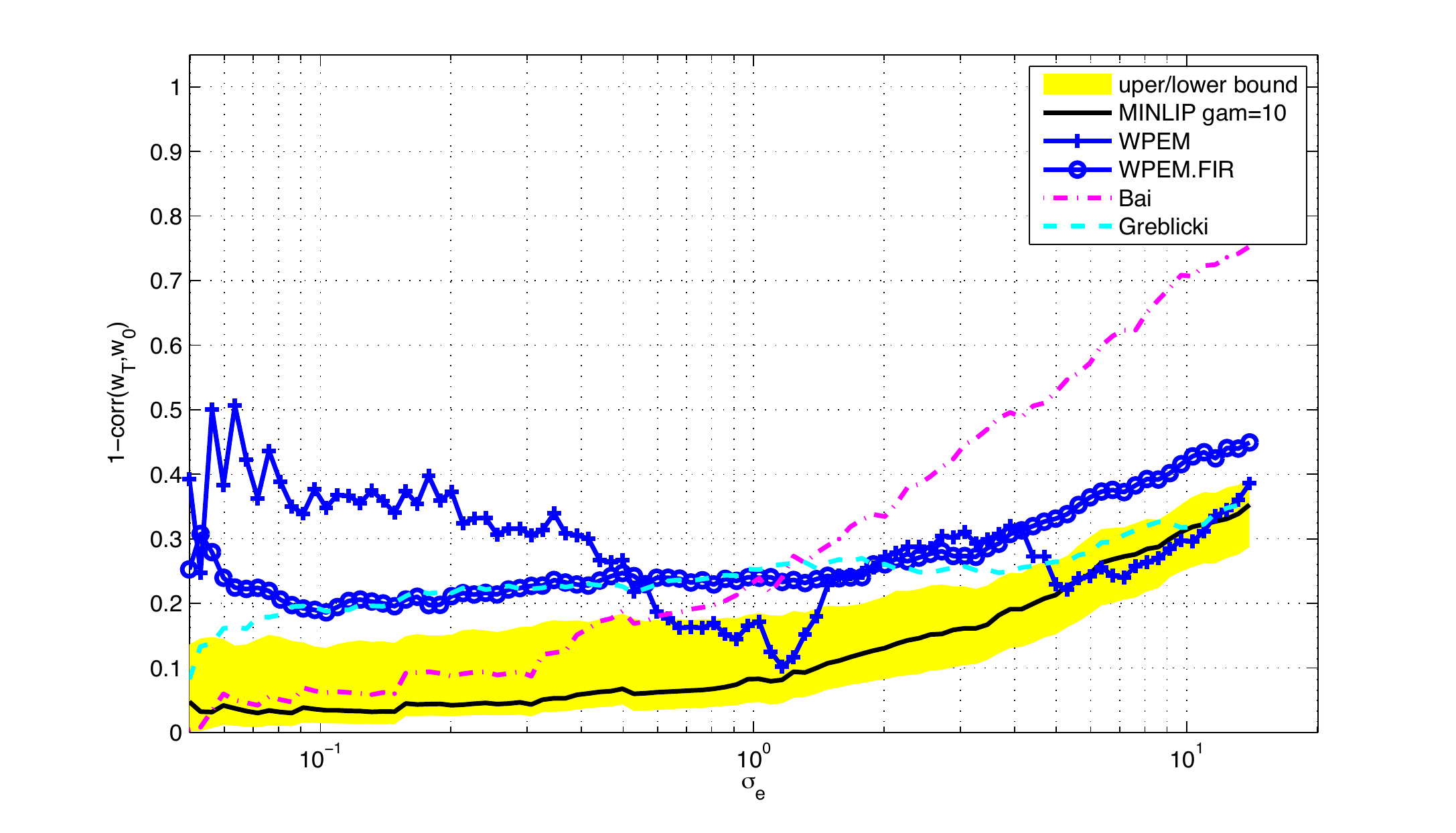}
   	\caption{\em
		Results of the third experiment using noisy data generated from 
		a monotone Wiener nonlinearity as in eq. (\ref{eq.f0}) and systems $H_0$ as in (\ref{eq.H1}).
		Performance expressed as correlations of $H_0$ and $\hat{H}_T$ as in (\ref{eq.corr}) are displayed 
		for a sample of sizes $T=500$. The amount of noise varies from $\sigma_e=0.1$ to $\sigma_e=3$,
		where the linear system $H_0$ is rescaled such that $\sigma_z=1$ (corresponding with SNR from 10 to 0.66).
		MINLIP was implemented with a fixed $\gamma=10$, 
		and can be seen to outperform other methods especially when the SNR is relatively high.
	}
   \label{fig.perf3}
\end{figure}

\begin{figure}[htbp] 
	\includegraphics[width=3in]{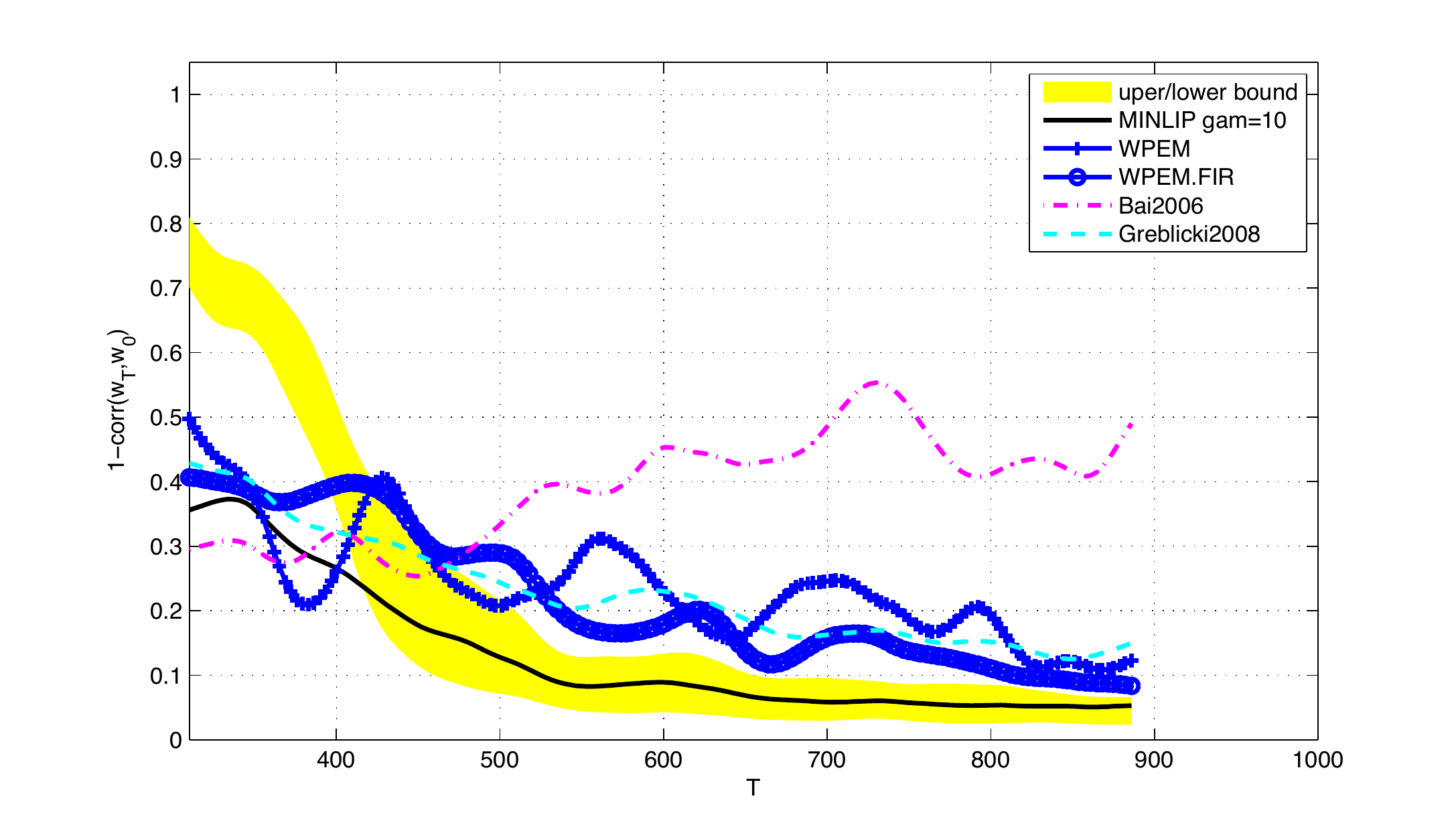}
   	\caption{\em
		Results of the last experiment using noisy data generated from 
		a monotone Wiener nonlinearity as in eq. (\ref{eq.f0}) and systems $H_0$ as in (\ref{eq.H1}), and SNR 
		equal to 3. Performances of the different estimators are 
		expressed as $1-$correlations of $H_0$ and $\hat{H}_T$ as in (\ref{eq.corr}).
		The sample sizes ranges from $T=210$ to $T=1000$, using a FIR over-parametrization of $d=200$.
	}
   \label{fig.perf4}
\end{figure}

\section{DISCUSSION}

This paper studies how MINLIP works for identification 
of monotone Wiener systems. Theoretical as well as 
empirical evidence indicates the use of the estimate despite 
its unconventional groundings.
Especially, one of the points is that this method based on model complexity control
can handle FIR overparameterizations of the linear subsystem quite efficiently,
implementing implicitly model order- and delay- estimation during the identification task.
The crux of the method is to place model complexity control in the centre of the identification task. 
The hope is that this line of thinking provides novel ideas which are useful in the 
design and analysis of identification algorithms for more general nonlinear systems.
A main open question is a theoretical study of the influence 
of noise in the almost consistency result.


\begin{thebibliography}{10}

\bibitem{bai2002}
E.W. Bai.
\newblock {A blind approach to the Hammerstein-Wiener model identification* 1}.
\newblock {\em Automatica}, 38(6):967--979, 2002.

\bibitem{bai2008}
E.W. Bai and J.~Reyland.
\newblock {Towards identification of Wiener systems with the least amount of a
  priori information on the nonlinearity}.
\newblock {\em Automatica}, 44(4):910--919, 2008.

\bibitem{billings1978}
S.A.~Billings and SY~Fakhouri.
\newblock {Identification of a class of nonlinear systems using correlation
  analysis}.
\newblock In {\em Institution of Electrical Engineers, Proceedings}, volume
  125, pages 691--697, 1978.

\bibitem{goodwin1984}
G.C. Goodwin and K.S. Sin.
\newblock {\em {Adaptive filtering prediction and control}}.
\newblock Prentice-Hall Englewood Cliffs, NJ, 1984.

\bibitem{greblicki2008}
W.~Greblicki and M.~Pawlak.
\newblock {\em {Non-Parametric System Identification}}.
\newblock Cambridge University Press, 2008.

\bibitem{hagenblad99a}
Anna Hagenblad.
\newblock {\em Aspects of the Identification of Wiener Models}.
\newblock PhD thesis, Department of Electrical Engineering, Link{\"{o}}ping
  University, Link\"oping, Sweden, Nov 1999.

\bibitem{ljung1977}
L.~Ljung.
\newblock {Analysis of recursive stochastic algorithms}.
\newblock {\em IEEE transactions on automatic control}, 22(4):551--575, 1977.

\bibitem{ljung87}
L.~Ljung.
\newblock {\em System Identification, Theory for the User}.
\newblock Prentice Hall, 1987.

\bibitem{pelckmans05_hamcdc}
K.~Pelckmans, I.~Goethals, J.A.K. Suykens, and B.~De Moor.
\newblock On model complexity control in identification of hammerstein systems.
\newblock In {\em the 44th IEEE conference on Decision and Control, and the
  European Control Conference (CDC-EEC 2005)}, Sevilla, Spain, 2005.

\bibitem{pelckmans10a}
K.~Pelckmans, T.~Van~Waterschoot, and J.A.K. Suykens.
\newblock Efficient adaptive filtering for smooth linear fir models.
\newblock In {\em Internal Report 10-60, ESAT-SISTA, K.U.Leuven, Belgium,
  submitted}. 2010.

\bibitem{soderstrom1989}
T.~S\"oderstrom and P.~Stoica.
\newblock {System identification}, 1989.

\bibitem{tsumura2002}
K.~Tsumura and J.~Maciejowski.
\newblock {\em {Optimal quantization of signals for system identification}}.
\newblock University of Cambridge, Department of Engineering, 2002.

\bibitem{vanbelle09c}
V.~Van~Belle, K.~Pelckmans, J.A.K. Suykens, and Van~Huffel S.
\newblock Learning transformation models for ranking and survival analysis,
  {\em submitted}.
\newblock 2009.

\bibitem{voros2001}
J.~Voros.
\newblock {Parameter identification of Wiener systems with discontinuous
  nonlinearities}.
\newblock {\em Systems and Control Letters}, 44(5):363--372, 2001.

\bibitem{westwick1996}
D.~Westwick and M.~Verhaegen.
\newblock {Identifying MIMO Wiener systems using subspace model identification
  methods}.
\newblock {\em Signal Processing}, 52(2):235--258, 1996.

\bibitem{wigren1994}
T.~Wigren.
\newblock {Convergence analysis of recursive identification algorithms basedon
  the nonlinear Wiener model}.
\newblock {\em IEEE Transactions on Automatic Control}, 39(11):2191--2206,
  1994.

\bibitem{wigren1995}
T.~Wigren.
\newblock {Approximate gradients, convergence and positive realness in
  recursive identification of a class of non-linear systems}.
\newblock {\em International Journal of Adaptive Control and Signal
  Processing}, 9(4), 1995.

\bibitem{wigren1998}
T.~Wigren.
\newblock {Adaptive filtering using quantized output measurements}.
\newblock {\em IEEE transactions on signal processing}, 46(12):3423--3426,
  1998.

\bibitem{zhang2006}
Q.~Zhang, A.~Iouditski, and L.~Ljung.
\newblock {Identification of Wiener system with monotonous nonlinearity}.
\newblock In {\em Proceedings of IFAC Symposium on System Identification}, page
  166-171, Newcastle, Australia, 2006.

\end{thebibliography}

\end{document}